\pdfoutput=1

\documentclass[11pt]{article}

\PassOptionsToPackage{dvipsnames}{xcolor}

\usepackage[preprint]{acl}

\usepackage{times}
\usepackage{latexsym}

\usepackage[T1]{fontenc}

\usepackage[utf8]{inputenc}

\usepackage{microtype}
\usepackage{inconsolata}

\usepackage{graphicx}

\usepackage{CJKutf8}

\usepackage{url}
\usepackage{xspace}
\usepackage{enumitem}
\usepackage{booktabs}
\usepackage{nicefrac}
\usepackage{booktabs}
\usepackage{multirow}
\usepackage{amsmath}
\usepackage{amssymb}
\usepackage{amsfonts}
\usepackage{amsthm}
\usepackage{thmtools}
\usepackage{mathtools}
\usepackage{bbm}
\usepackage{bm}
\usepackage{scalerel}

\usepackage{adjustbox}

\usepackage{subcaption}
\usepackage{wrapfig}
\usepackage{thm-restate}

\usepackage[linesnumbered,ruled,vlined]{algorithm2e}
\usepackage[most]{tcolorbox}

\definecolor{darkblue}{rgb}{0, 0, 0.5}

\usepackage{tabularray}
\usepackage{cellspace}
\usepackage{setspace}
\newcommand\cincludegraphics[2][]{\raisebox{-0.28\height}{\includegraphics[#1]{#2}}}

\makeatletter                             %
    \DeclareRobustCommand*{\escapeus}[1]{%
    \begingroup\@activeus\scantokens{#1\endinput}\endgroup}
    \begingroup\lccode`\~=`\_\relax
   \lowercase{\endgroup\def\@activeus{\catcode`\_=\active \let~\_}}
\makeatother
\newcommand{\myemph}[1]{\textsf{{\escapeus{#1}}}}

\usepackage{amssymb}%
\usepackage{pifont}%
\newcommand{\xmark}{\ding{55}}%

\usepackage{cleveref}

\definecolor{colourcharacter}{RGB}{70,120,255}
\definecolor{coloursubword}{RGB}{220,50,140}
\definecolor{colourbase}{RGB}{255,120,20}
\definecolor{boxgray}{RGB}{140,140,140}
\definecolor{macrocolour}{RGB}{128,0,128}

\newcommand{\mysubword}[2]{\newcommand{#1}{{\color{coloursubword}#2}}}
\newcommand{\myfunction}[2]{\newcommand{#1}{{\color{colourbase}#2}}}

\newcommand{\alphabet}{{\color{colourcharacter}\Sigma}}
\newcommand{\character}{{\color{colourcharacter}c}}
\newcommand{\characters}{{\color{colourcharacter}\mathbf{c}}}
\newcommand{\dataset}{\mathcal{D}}

\newcommand{\vocab}{{\color{coloursubword}\mathcal{S}}}
\newcommand{\subword}{{\color{coloursubword}s}}
\newcommand{\subwords}{{\color{coloursubword}\mathbf{s}}}
\newcommand{\subwordt}[1]{{\color{coloursubword}s_{\textcolor{black}{#1}}}}

\newcommand{\merge}{{\color{coloursubword}m}}
\newcommand{\merges}{{\color{coloursubword}\mathbf{m}}}

\mysubword{\tokeniser}{\mathbb{T}}

\newcommand{\tokenise}{{\color{colourbase}\mathtt{tok}}}
\newcommand{\detokenise}{{\color{colourbase}\mathtt{detok}}}
\newcommand{\directtoken}{{\color{colourbase}\mathtt{tok}_{\Rightarrow}}}
\newcommand{\bottomuptoken}{{\color{colourbase}\mathtt{tok}_{\uparrow}}}

\myfunction{\objectivefunc}{\mathfrak{G}}
\newcommand{\objectivefunccomp}{\objectivefunc_{\mathrm{comp}}}
\newcommand{\objectivefuncll}{\objectivefunc_{\mathrm{ll}}}
\myfunction{\qualityfunc}{\mathfrak{J}}

\newcommand{\tokenisespace}{\mathcal{T}}
\newcommand{\vocabsize}{K}

\newcommand{\stringequiv}{\overset{\circ}{=}}
\newcommand{\defeq}{\stackrel{\texttt{\tiny def}}{=}}

\newcommand{\charstring}[1]{{\color{colourcharacter}#1}}
\newcommand{\subwordstring}[1]{{\color{coloursubword}#1}}

\newcommand{\defn}[1]{\textbf{#1}}

\newcommand{\totcnt}{N^{\scaleto{\dataset}{4pt}}}
\newcommand{\replacementseq}{\subwords^{\mathrm{rep}}}

\crefname{section}{\S}{\S\S}
\Crefname{section}{\S}{\S\S}
\crefname{table}{Tab.}{Tabs.}
\crefname{figure}{Fig.}{Figs.}
\crefname{algorithm}{Alg.}{Algs.}
\crefname{equation}{Eq.}{Eqs.}
\crefname{appendix}{App.}{Apps.}
\crefname{defin}{Definition}{Definitions}
\crefname{lemma}{Lemma}{Lemmas}
\crefname{theorem}{Theorem}{Theorems}
\crefname{proposition}{Proposition}{Propositions}
\crefname{corollary}{Corollary}{Corollaries}
\crefformat{section}{\S#2#1#3}
\crefformat{footnote}{#2\footnotemark[#1]#3}

\DeclareMathOperator*{\argmin}{\mathrm{argmin}}
\DeclareMathOperator*{\argmax}{\mathrm{argmax}}

\newcommand{\countfunction}{n}

\newcommand{\np}{\texttt{NP}}

\myfunction{\mergefunc}{\texttt{merge}}

\newcommand{\BPE}{\texttt{BPE}\xspace}
\newcommand{\UnigramLM}{\texttt{UnigramLM}\xspace}
\newcommand{\GreedyLL}{\texttt{BottomUpLL}\xspace}
\newcommand{\CompMax}{\texttt{TopDownComp}\xspace}

\newcommand{\btheta}{\boldsymbol{\theta}}
\newcommand{\ptheta}{p_{\btheta}}
\newcommand{\opttokenise}{\tokenise_{\mathtt{opt}}}

\newcommand{\greedyalg}{bottom-up\xspace}
\newcommand{\pruningalg}{top-down\xspace}

\newcommand{\subwordbad}{\subword_{\text{\xmark}}}

\newcommand{\mergegood}{\merge^\star}
\newcommand{\bpb}{\textsc{BPB}\xspace}

\title{Objective vs.\ Search: Decomposing What Makes a Good Tokeniser}

\newcommand{\makesf}[1]{\textsf{{{#1}}}}
\newcommand{\ethemailadresstwo}[1]{\href{mailto:#1}{\makesf{#1}}}
\newcommand{\ethemailadress}[1]{\href{mailto:#1}{\makesf{#1}}}
\newcommand{\epflemailadress}[1]{\href{mailto:#1}{\makesf{#1}}}

\author{Ahmetcan Yavuz$^1$ \quad Clara Meister$^2$ \quad Tiago Pimentel$^1$ \\
  $^1$ETH Z\"urich, $^2$EPFL 
  \\
   \ethemailadresstwo{ayavuz@ethz.ch},\,
   \epflemailadress{clara.meister@epfl.ch},\,
  \ethemailadress{tiago.pimentel@inf.ethz.ch}
    \\
    \begin{tblr}{colspec = {Q[c,m] Q[c,m]}, colsep=10pt, stretch=0}
        \cincludegraphics[width=1.1em, keepaspectratio]{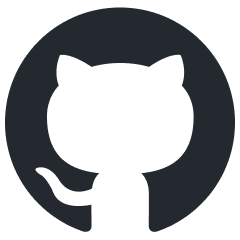} {\fontsize{11pt}{11.5pt}\selectfont\href{https://github.com/Ahmetcanyvz/comp-vs-like}{\myemph{Ahmetcanyvz/comp-vs-like}}}
        & \cincludegraphics[width=1.em, keepaspectratio]{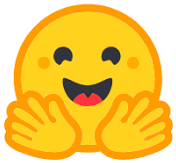} {\fontsize{11pt}{11.5pt}\selectfont\href{https://huggingface.co/collections/AhmetcanYvz/comp-vs-like}{\myemph{Ahmetcanyvz/comp-vs-like}}}
    \end{tblr}
  }

\begin{document}

\maketitle
\begin{abstract}
Two dominant tokenisation algorithms are used by modern language models: byte-pair encoding (\BPE) and \UnigramLM.
These differ along two orthogonal axes: their \emph{optimisation objective} (compression vs.\ log-likelihood) and their \emph{search procedure} (bottom-up merging vs.\ top-down pruning).
Existing comparisons confound these axes, making it unclear whether their observed differences stem from \emph{what} is being optimised vs. \emph{how} it is being optimised.
We disentangle the two by introducing two new tokenisation algorithms that complete this 2$\times$2 design space: \GreedyLL, a bottom-up likelihood-based tokeniser, and \CompMax, a top-down compression-based tokeniser.
We train language models with tokenisers produced by each algorithm, varying: model size, vocabulary sizes, and domain (English-only vs.\ multilingual).
Evaluating models on bits-per-byte, we find that the search procedure---not the objective---is the dominant factor: bottom-up tokenisers consistently achieve lower bits-per-byte in most settings.
Evaluating models on the BLiMP task, however, shows no consistent relationship between design choice and performance.
Overall, our results disentangle the effect of tokeniser design choices on language modelling performance, offering concrete guidance for their more principled construction.
\end{abstract}

\section{Introduction}

Before a language model (LM) processes text, its raw character string is mapped to a token string, which defines the model's input.
This mapping is performed by a \defn{tokeniser}, a foundational component of modern language modelling pipelines, used in virtually all state-of-the-art models \citep{qwenteam2026qwen35omnitechnicalreport, openai2025gptoss120bgptoss20bmodel, kimiteam2026kimik25visualagentic}.
A tokeniser is defined by several attributes.
Some, such as the vocabulary size, are set by the practitioner while others, such as the vocabulary itself, are learned from data by a \defn{tokeniser learning algorithm}.
Two such algorithms account for most tokenisers in current use:
 (i) byte-pair encoding (\BPE), which learns a vocabulary that maximises the compression of a training corpus \citep{gage1994new,sennrich-etal-2016-neural}; and (ii) \UnigramLM{}, which learns a vocabulary that maximises the unigram log-likelihood of a training corpus \citep{kudo-2018-subword}.

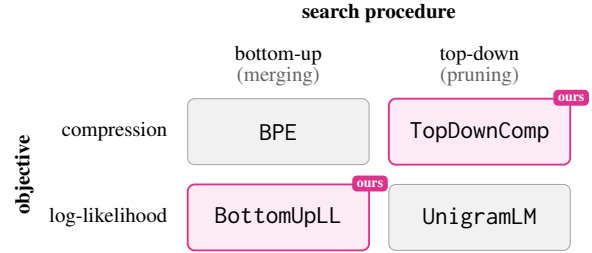
\begin{figure}[t]
  \centering
  \begin{adjustbox}{max width=\columnwidth}
  \begin{tikzpicture}[
      cell/.style={rounded corners=3pt, minimum width=2.75cm,
        minimum height=1.0cm, align=center, inner sep=3pt},
      existing/.style={cell, draw=boxgray!80, fill=boxgray!12},
      ours/.style={cell, draw=coloursubword, thick, fill=coloursubword!10},
      colhead/.style={align=center, font=\small},
      rowhead/.style={align=right, font=\small, anchor=east},
      axis/.style={font=\small\bfseries},
      badge/.style={font=\tiny\bfseries, text=white, fill=coloursubword,
        rounded corners=2pt, inner sep=2pt},
    ]
    \node[existing] (bpe) at (0,0)       {\BPE};
    \node[ours]     (cm)  at (3.05,0)    {\CompMax};
    \node[ours]     (gll) at (0,-1.3)    {\GreedyLL};
    \node[existing] (ulm) at (3.05,-1.3) {\UnigramLM};
    \node[badge] at (cm.north east)  {ours};
    \node[badge] at (gll.north east) {ours};
    \node[colhead] at (0,1.0)    {bottom-up\\[-1pt]{\footnotesize\color{black!60}(merging)}};
    \node[colhead] at (3.05,1.0) {top-down\\[-1pt]{\footnotesize\color{black!60}(pruning)}};
    \node[axis] at (1.525,1.8) {search procedure};
    \node[rowhead] at (-1.55,0)    {compression};
    \node[rowhead] at (-1.55,-1.3) {log-likelihood};
    \node[axis, rotate=90] at (-3.85,-0.65) {objective};
  \end{tikzpicture}
  \end{adjustbox}
  \caption{Design choice factorisation for tokenisation.}
  \label{fig:two-by-two}
  \vspace{-10pt}
\end{figure}

Despite their widespread usage, though, it remains unclear what makes a good tokeniser.
\BPE{} and \UnigramLM{} are used as the default algorithms; yet their design choices have received little direct examination, and upon closer inspection, confound many of the existing comparisons.
As a concrete example, the results of \citet{schmidt-etal-2024-tokenization}---that language models trained with \BPE{} often perform better than those trained with \UnigramLM{}---are commonly read as evidence that compression is a better \defn{optimisation objective} than unigram log-likelihood.
Optimisation objectives, though, are only one facet of a tokeniser learning algorithm.

\BPE and \UnigramLM also differ in the \defn{search procedure} used to maximize the objective.\footnote{The optimisation problem that such algorithms must solve is, in general, \np-hard \citep{kozma2024theoretical,whittington-etal-2025-tokenisationnpc,kastreva2026tokenisation}, so tokeniser learning algorithms typically approximate a solution with a heuristic search procedure.\looseness=-1}
Given a dataset, \BPE{} uses a bottom-up search, greedily merging symbols to maximise compression.
\UnigramLM{} relies on a top-down search, starting from an oversized vocabulary and iteratively pruning tokens to maximise log-likelihood.
\BPE{} and \UnigramLM{} therefore differ along two axes at once: the optimisation objective and the search procedure.
A comparison between the two cannot attribute the difference in performance to either axis alone.

In this work, we disentangle these two axes.
We introduce \GreedyLL, a bottom-up likelihood-based tokeniser that mirrors \BPE{}'s merge procedure but optimises corpus log-likelihood instead of compression.
We also introduce \CompMax, a top-down compression-based tokeniser that follows \UnigramLM{}'s pruning procedure but optimises compression.
Together with \BPE{} and \UnigramLM{}, these methods form a $2\times2$ study crossing objective (compression vs.\ log-likelihood) and search procedure (\greedyalg vs.\ \pruningalg); see \cref{fig:two-by-two}.

We train language models with tokenisers produced by all four algorithms, evaluating them in English-only and multilingual settings.
We analyse these models with a combination of extrinsic
and intrinsic evaluations.
Notably, our results show that search procedure has a stronger influence on bits-per-byte (BPB) than the training objective, with bottom-up tokenisers achieving the best BPB scores in nearly all experimental conditions (\BPE{} outperforms \CompMax{}, and \GreedyLL{} outperforms \UnigramLM{}).
This ordering does not carry over to grammaticality judgements (on a BLiMP task), though, where neither the objective nor the search procedure separates the tokenisers consistently.
Interestingly, at small vocabulary settings, however, the objective still matters, with likelihood-based tokenisers outperform compression-based ones.
Overall, our results highlight the importance of both optimisation objective and search procedure to tokeniser training algorithm design.

\section{Tokenisation}\label{sec:background}

Language modelling starts from raw text, which can be represented as \defn{character-strings} $\characters \in \alphabet^*$; 
these are finite sequences of characters $\characters = \charstring{\character_1 \character_2 \dots \character_{|\characters|}}$ over alphabet $\alphabet$.\footnote{
As usual, $\alphabet^*$ denotes the Kleene star of $\alphabet$, while $\alphabet^+$ denotes the set of all non-empty strings over $\alphabet$.
\looseness=-1}
A tokeniser's job is then to \emph{segment} these character-strings into \defn{token-strings} $\subwords \in \vocab^*$: sequences $\subwords = \subwordstring{\langle \subwordt{1}, \subwordt{2}, \dots, \subwordt{|\subwords|} \rangle}$, where each symbol $\subwordt{t}$ is called a token and represents a non-empty character span.
Whenever $\subwords$ is a segmentation of $\characters$, we say these strings are equivalent, which we denote as $\characters \stringequiv \subwords$, i.e.:\footnote{We restrict our discussion to \emph{lossless} tokenisers here.}
\begin{align}
    \characters \stringequiv \subwords
    \iff
    \characters = \subwordt{1} \circ \subwordt{2} \circ \dots \circ \subwordt{|\subwords|},
\end{align}
Notably, tokenisers are usually not allowed to output any possible segmentation of $\characters$, being instead constrained to segment it using only the tokens in a finite set called its \defn{vocabulary} $\vocab \subset \alphabet^+$.
To ensure that any character-string can be represented as tokens, we enforce $\alphabet \subseteq \vocab$.

How do we convert between character- and token-strings, though? Given just a vocabulary $\vocab$, there are multiple different ways a character-string could be mapped to items in that vocabulary. 
For example, if
$\vocab = \{\subwordstring{a}, \subwordstring{aa}, \subwordstring{aaa}\},$
then the character-string $\characters = \charstring{aaa}$ may be encoded as either
$\subwordstring{\langle a, aa \rangle}$,
$\subwordstring{\langle aa, a \rangle}$, or
$\subwordstring{\langle a, a, a \rangle}$, which we refer to as \defn{segmentations}.
Consequently, a tokeniser is defined not only by its vocabulary, but also by \emph{how} it maps character-strings to segmentations. 
This mapping is the job of an \defn{encoding function}, $\tokenise : \alphabet^* \to \vocab^*$, which segments the original character-string into tokens; by definition,
we have thus that $\characters \stringequiv \tokenise(\characters)$. 
Finally, a tokeniser also contains a \defn{decoding function}, $\detokenise : \vocab^* \to \alphabet^*$, which converts token-strings back into characters, being typically defined as the concatenation of each token's characters: $\detokenise(\subwords) \defeq \subwordt{1} \circ \subwordt{2} \circ \dots \circ \subwordt{|\subwords|}$.
Formally, we thus define a tokeniser as the tuple
$\tokeniser \defeq \langle \vocab, \detokenise, \tokenise \rangle$.

\newcommand{\countsnew}[2]{\countfunction_{#1}^{{\scaleto{#2}{4pt}}}}
\newcommand{\tokendelta}[2]{\Delta_{#1}^{#2}}

\section{Learning a Tokeniser: Optimisation Objective vs.\ Search Procedure}
\label{sec:objective}

How do we select such a tokeniser? Here, we cast the learning of a tokeniser as the combination of two design components: an objective function and a search procedure.

\subsection{Objective Functions}

Given a tokeniser $\tokeniser$ and a dataset $\dataset = \{\characters_m\}_{m=1}^{M}$, an objective function $\objectivefunc$ assigns a score to the tokeniser.
Different objectives encode different notions of what makes a tokeniser useful.
We focus on two standard objectives used by tokeniser learning algorithms: compression and log-likelihood.

Under the \defn{compression} objective, the quantity to be minimised is the total number of tokens that the tokeniser produces on the dataset:
\begin{align}
    \objectivefunccomp(\tokenise,\dataset)
    \defeq
    \sum_{\characters \in \dataset} |\tokenise(\characters)|,
\end{align}
This is appealing because shorter token strings improve effective context usage and reduce the number of model steps needed to process a dataset.

Under the \defn{log-likelihood} objective, the minimised quantity is the negative log-probability that a given probabilistic model $\ptheta$ assigns the corpus:
\begin{align}
    \objectivefuncll(\tokenise,\dataset)
    \defeq
    -\sum_{\characters \in \dataset}
    \log \ptheta\bigl(\tokenise(\characters)\bigr).
\end{align}
where $\ptheta$'s parameters $\btheta$ are themselves optimised to model this tokeniser's text.
This objective is likewise appealing, since log-likelihood is also the training objective of the language models for which these tokenisers are designed---so a log-likelihood optimal tokeniser would, by construction, lead to better models.
Optimising log-likelihood directly, however, is intractable, as this objective requires fully training language models $\ptheta$ for its evaluation.

In practice, to make this objective tractable, we typically significantly restrict the class of models $\ptheta$.
Here, we follow \UnigramLM{} in restricting $\ptheta$ to unigram language models.
For a fixed encoding function $\tokenise$, let $\countsnew{\subword}{\dataset}$ denote the count of token $\subword$ in the tokenised corpus,
and let
$\totcnt \defeq \sum_{\subword \in \vocab} \countsnew{\subword}{\dataset}$
be the total number of tokens.
We then define $\ptheta$ as
\begin{equation}
    \ptheta(\subwords)
    =
    \prod_{t=1}^{|\subwords|}
    \ptheta(\subwordt{t}),
    \qquad
    \ptheta(\subword)
    =
    \frac{\countsnew{\subword}{\dataset}}{\totcnt}.
\end{equation}
Concretely, we will thus use \defn{unigram log-likelihood} as the objective here, as opposed to `full' log-likelihood, and, for the rest of this paper, log-likelihood stands for unigram log-likelihood.

\vspace{-4pt}
\subsection{Search Procedures}\label{sec:search_proc}

Given one of these objectives and a fixed vocabulary budget $\vocabsize$, we might expect a tokeniser to be chosen by directly optimising
\begin{align}
    \opttokenise
    =
    \argmin_{\tokenise \in \tokenisespace_{\vocabsize}}
    \objectivefunc(\tokenise,\dataset),
\end{align}
where $\tokenisespace_{\vocabsize}$ denotes whichever class of tokenisers is under consideration.
Unfortunately, this optimisation problem is computationally intractable; the compression variant is provably \np-hard \citep{kozma2024theoretical,whittington-etal-2025-tokenisationnpc,kastreva2026tokenisation}, and we suspect the log-likelihood variant to be as well.
In practice, heuristic \defn{search procedures} are therefore required.
Such procedures specify both a restricted class of tokenisers and a strategy for selecting an element from that class.
We focus on two standard search procedures: \greedyalg construction and \pruningalg pruning.

The \defn{\greedyalg} procedure builds a tokeniser from the bottom up: the vocabulary starts as the alphabet $\alphabet$ and grows one token at a time
through \emph{merge operations}. 
Given a merge $\merge = \langle \subword', \subword'' \rangle$, applying it to a token-string rewrites every (non-overlapping) occurrence of bigram $\subword', \subword''$ into the token $\subword^{\texttt{new}} = \subword' \circ \subword''$.
A merge list $\merges = \merge_1,\dots,\merge_k$ then defines an encoding function
$\bottomuptoken[\merges]$: starting from $\characters$ as a token-string over $\alphabet$ (one token per character), it applies $\merge_1,\dots,\merge_k$ in order. The final token-string is our tokenised text.  
We write the tokeniser class defined by this procedure as $\tokenisespace_{\vocabsize}^{\uparrow}
    =
    \bigl\{
    \bottomuptoken[\merges]
    \mid
    \merges \in (\alphabet^+ \times \alphabet^+)^\vocabsize
    \bigr\}$.
To select an element from this class, we then start from an empty merge list, and greedily add one merge per step to it, locally optimising an objective $\objectivefunc$:
\begin{align}
    \mergegood
    =
    \argmin_{\merge \in \alphabet^+ \times \alphabet^+}
    \objectivefunc(
        \bottomuptoken[\merges_{<k}, \merge],
        \dataset
    ),
    \label{eq:bottomup_search}
\end{align}
where $\bottomuptoken[\merges_{<k}, \merge]$ denotes the encoding function obtained by applying the current merge list $\merges_{<k}$ followed by the candidate $\merge$.
The selected merge $\mergegood$ is then appended to the merge list. 
We refer to a candidate merge's effect on the objective as its \defn{merge gain}, and write it $\tokendelta{\subword^1,\subword^2}{\objectivefunc(\cdot,\dataset)}$.

By contrast, the \defn{\pruningalg} procedure starts from a large candidate vocabulary $\vocab_0$, with $|\vocab_0| \gg \vocabsize$, and removes tokens until $|\vocab| = |\alphabet| + \vocabsize$.
Notably, this procedure does not rely on merge lists,
but instead commits to an \defn{objective-optimal} encoding function $\directtoken[\vocab]$, determined by the vocabulary alone.
Under compression, $\directtoken[\vocab](\characters)$ is the shortest valid segmentation of $\characters$ using tokens from $\vocab$.
Under log-likelihood, it is instead the most probable segmentation of $\characters$ under $\ptheta$.\footnote{Given a fixed $\ptheta$, both cases can be computed efficiently. For log-likelihood, however, 
the parameters of $\ptheta$ itself depend on $\tokenise$, creating a circular definition. In practice, the two are estimated jointly as we explain in \cref{sec:method_topdown}.}
Top-down's tokeniser class can then be written as:
$
    \tokenisespace_{\vocabsize}^{\downarrow}
    =
    \bigl\{
    \directtoken[\vocab]
    \mid
    \vocab \subset \alphabet^+,
    \alphabet \subseteq \vocab,
    |\vocab| = |\alphabet| + \vocabsize
    \bigr\}
$.
To select an item in this class, we then start from a large vocabulary $\vocab_0$, from which we iteratively prune the least critical tokens, quantified as the tokens whose removal has the least negative impact on the objective we're optimizing for:
\begin{align}
    \subwordbad
    =
    \argmin_{\subword \in \vocab_{k-1} \setminus \alphabet}
    \objectivefunc(
        \directtoken[\vocab_{k-1} \setminus \{\subword\}],
        \dataset
    ).
    \label{eq:topdown_search}
\end{align}
Here, we index pruning steps by $k$, mirroring the bottom-up case.
This yields $\vocab_k = \vocab_{k-1} \setminus \{\subwordbad\}$. 
Importantly, to ensure every string remains encodable, alphabet tokens are never removed; pruning then continues until $|\vocab| = |\alphabet| + \vocabsize$.\footnote{We note that, for efficiency reasons, in practice, multiple tokens are often pruned in batches.} 
Analogously, we refer to a candidate deletion's effect on the objective as its \defn{deletion cost}, written $\tokendelta{\subword}{\objectivefunc(\cdot,\dataset)}$.

\vspace{-4pt}
\section{Analysed Tokenisers: Old and New}
\label{sec:method}

Combining the objectives and search procedures above, we get a $2\times2$ design space for tokenisation algorithms (illustrated in \cref{fig:two-by-two}).
Existing tokeniser learning algorithms, namely \BPE and \UnigramLM, cover only two of these cells.
Here, we introduce algorithms that instantiate the two other cells: \GreedyLL{}, a \greedyalg likelihood-based method, and \CompMax{}, a \pruningalg compression-based method.\footnote{Our new methods are closely related to, e.g., WordPiece or PathPiece. We discuss this prior work in detail in \cref{sec:prior_work}.\looseness=-1}
Before any of these methods can be run, however, a loose end remains.

In the previous section, we describe the search procedures as greedy: at each step, it picks the merge $\mergegood$ or deletion $\subwordbad$ which (locally) optimises an objective $\objectivefunc(\cdot,\dataset)$.
It did not specify, however, how this selection is performed.
Taken literally, computing the $\argmin$ operations in \cref{eq:bottomup_search,eq:topdown_search} would require running the objective function $\objectivefunc$ once per candidate to find the optimal one---a computationally impractical operation.
Rather than recomputing $\objectivefunc$ every time, most tokenisation algorithms work with \emph{incremental scores} instead: 
\begin{subequations}
\begin{align}
    \mergegood &=
    \argmax_{\langle \subword^1, \subword^2 \rangle \in \alphabet^+ \times \alphabet^+} \tokendelta{\subword^1,\subword^2}{\objectivefunc(\cdot, \dataset)} \\
    \subwordbad &=
    \,\,\,\,\,\, \argmin_{\subword \in \vocab_{k-1}\setminus\alphabet} 
    \,\,\,\,\,\,
    \tokendelta{\subword}{\objectivefunc(\cdot,\dataset)}
    \label{eq:topdown_delta_search}
\end{align}
\end{subequations}
Notably, these equations are equivalent to \cref{sec:search_proc}'s, as the pre-update objective values (e.g., $\objectivefunc(\bottomuptoken[\merges_{<k}], \dataset)$)  do not depend on the candidates, and are thus constant.
Each search procedure then derives a way to compute (or approximate) these values efficiently, as we show next.

\subsection{Bottom-up tokenisers: \BPE\ and \GreedyLL}
\label{sec:method_bottomup}

For a bottom-up method, each candidate is scored according to its merge gain $\tokendelta{\subword^1,\subword^2}{\objectivefunc(\cdot,\dataset)}$.
We now show how this score can be computed efficiently for the two objectives, starting with compression.

\begin{restatable}{lemma}{equivbpefreq}
\label{lemma:equivbpefreq}
    The change in compression due to merging a token-pair can be computed as:
    \begin{align}
        \tokendelta{\subword^1,\subword^2}{\objectivefunccomp(\cdot, \dataset)} = \countsnew{\subword^1, \subword^2}{\dataset}
    \end{align}
    where $\countsnew{\subword^1, \subword^2}{\dataset}$ denotes the number of non-overlapping $\langle \subword^1, \subword^2 \rangle$ sequences in $\dataset$.
\end{restatable}

\begin{proof}[Proof sketch]
    The full proof is in \Cref{app:equivbpefreq}.
    Replacing an occurrence of the token-pair $\merge = \langle \subword^1, \subword^2 \rangle$ with the merged token $\subword^1 \circ \subword^2$ saves one token.
    Further, the number of non-overlapping occurrences of a token-pair is its number of possible merges.
    Merging the pair therefore shortens the corpus by exactly $\countsnew{\subword^1, \subword^2}{\dataset}$ tokens, which is its gain under $\objectivefunccomp$.
\end{proof}

\Cref{lemma:equivbpefreq} gives the standard \BPE{} merge rule: the compression-optimal merge is the most (non-overlappingly) frequent token-pair.
In practice, the token-pair counts $\countsnew{\subword^1,\subword^2}{\dataset}$ are computed once and then updated incrementally after each merge, rather than recomputed from scratch (see \cref{app:update_counts}), which allows us to run \BPE efficiently.
We can derive an analogous equivalence for \GreedyLL{}.

\begin{restatable}{lemma}{equivgreedyll}
\label{lemma:equivgreedyll}
    The change in log-likelihood due to merging a token-pair can be computed as:\footnote{
For the case $\subword^1=\subword^2$, the same calculation applies albeit with a modified count update:
each merge consumes two occurrences of the same token $\subword^1$ and thus its count is reduced by $2\countsnew{\subword^1,\subword^1}{\dataset}$ instead; other parts of the update remain unchanged.}\looseness=-1
    \begin{align}
    &\tokendelta{\subword^1,\subword^2}{\objectivefuncll(\cdot, \dataset)} \!= 
    \\
    &\,\, 
    (\countsnew{\subword^1}{\dataset} - \countsnew{\subword^1,\subword^2}{\dataset})
    \log(\countsnew{\subword^1}{\dataset} - \countsnew{\subword^1,\subword^2}{\dataset}) 
    - \countsnew{\subword^1}{\dataset} \log \countsnew{\subword^1}{\dataset} \nonumber\\
    &\,\, + (\countsnew{\subword^2}{\dataset} - \countsnew{\subword^1,\subword^2}{\dataset})
    \log(\countsnew{\subword^2}{\dataset} - \countsnew{\subword^1,\subword^2}{\dataset}) \nonumber
    - \countsnew{\subword^2}{\dataset} \log \countsnew{\subword^2}{\dataset} \nonumber\\
    &\,\, + \countsnew{\subword^1,\subword^2}{\dataset}
    \log \countsnew{\subword^1,\subword^2}{\dataset} \nonumber\\
    &\,\, - (\totcnt - \countsnew{\subword^1,\subword^2}{\dataset})
    \log(\totcnt - \countsnew{\subword^1,\subword^2}{\dataset}) \nonumber + \totcnt\log \totcnt. \nonumber
    \end{align}
\end{restatable}

\begin{proof}[Proof sketch]
    The full proof is in \Cref{app:greedyllgain}.
    Note that, if a token changes from count $\countfunction_1$ to count $\countfunction_2$, its contribution to the log-likelihood objective changes by: $\countfunction_2 \log \countfunction_2 - \countfunction_1 \log \countfunction_1$.
    The lemma then follows trivially, noting that:
    (i) the old tokens $\subword^1$ and $\subword^2$ change from, e.g., frequency $\countsnew{\subword^1}{\dataset}$ to $(\countsnew{\subword^1}{\dataset} - \countsnew{\subword^1,\subword^2}{\dataset})$;
    (ii) the new token $\subword^1 \circ \subword^2$ changes from frequency 0 to $\countsnew{\subword^1,\subword^2}{\dataset}$.
    Further, the total count of tokens serves as a normalisation factor in all tokens, and thus enters the log-likelihood with the opposite sign; its contribution changes from $\totcnt$ to $\totcnt - \countsnew{\subword^1,\subword^2}{\dataset}$.
\end{proof}

\GreedyLL{} can thus also be implemented efficiently, maintaining local count statistics and updating only affected candidate scores after each merge.
Interestingly, an analysis of the merge gain definition in \cref{lemma:equivgreedyll} shows that, for a token-pair to be selected: (i) the pair should occur often enough to matter, but (ii) it should also be favoured when its tokens co-occur more systematically than expected from their individual frequencies.

\vspace{-4pt}
\subsection{Top-down tokenisers: \UnigramLM\ and \CompMax}
\label{sec:method_topdown}

For \pruningalg methods, each candidate's score is given by the deletion cost $\tokendelta{\subword}{\objectivefunc(\cdot,\dataset)}$.
Unfortunately, this cost $\tokendelta{\subword}{\objectivefunc(\cdot,\dataset)}$ is non-trivial to compute.
This is due to the nature of the encoding function $\directtoken[\vocab]$ used by top-down algorithms.
Unlike in the bottom-up case, deleting a token here can change the preferred segmentation of an entire character-string, not only the local decomposition of the deleted token.\footnote{For example, under a compression objective, suppose
$\vocab \!=\! \alphabet \cup \{\subwordstring{abc}, \subwordstring{def}, \subwordstring{cdef}\}$.
Then, string $\charstring{abcdef}$ is segmented as
$\subwordstring{\langle abc, def\rangle}$.
If $\subwordstring{abc}$ is deleted, however, the shortest segmentation under the reduced vocabulary becomes
$\subwordstring{\langle a,b,cdef\rangle}$, which also changes the previously used token $\subwordstring{def}$.}
For the log-likelihood objective, the history gets worse.
The segmentations produced by $\directtoken[\vocab]$ depend on $\ptheta$, which is itself estimated from the token counts.
Deleting a token changes those counts, which changes $\ptheta$, which in turn changes how strings are segmented, which again changes the counts.
Computing a deletion's exact cost would thus require running this loop to convergence for each candidate token.\looseness=-1

To approximate $\tokendelta{\subword}{\objectivefunc(\cdot,\dataset)}$, we thus rely on a \defn{local replacement approximation}: when scoring the deletion of $\subword$, we keep the rest of the current model fixed and estimate the loss by only directly replacing occurrences of $\subword$ with an alternative segmentation under $\vocab\setminus\{\subword\}$.\footnote{Note that this introduces an asymmetry between our two search procedures: the \greedyalg{} merge gains in the previous section are exact, whereas the \pruningalg{} deletion costs are only approximate.
We emphasise that this asymmetry is not a design choice on our part, but an inherent property of \pruningalg{} search procedures, being itself present in \UnigramLM{}. 
We return to its implications in our Limitations section.}
For log-likelihood, this then yields the following result.

\begin{restatable}{lemma}{unigramdeletioncost}
\label{lemma:unigramdeletioncost}
Under a local replacement approximation, the estimated increase in the negative log-likelihood objective due to deleting a token $\subword \in \vocab \setminus \alphabet$ can be approximated as:\footnote{In practice, when computing this update, \UnigramLM keeps segmentations latent and leverages $\ptheta$ to compute $\countsnew{\subword}{\dataset}$ as the  \emph{expected} count of $\subword$, instead of its exact count under $\directtoken$.\looseness=-1}
\begin{align}
    \tokendelta{\subword}{\objectivefuncll(\cdot,\dataset)}
    \approx
    \countsnew{\subword}{\dataset}
    \left(
    \log \ptheta(\subword)
    -
    \log \ptheta(\replacementseq_{\subword})
    \right),
    \label{eq:unigram_deletion_cost}
\end{align}
where $\ptheta$ is the unigram model at step $k-1$, and $\replacementseq_{\subword}$ is $\subword$'s log-probability--optimal replacement segmentation, i.e.,
$\replacementseq_{\subword}
    \defeq \directtoken[\vocab_{k-1}\setminus\{\subword\}](\subword) $.
\end{restatable}
\begin{proof}[Proof sketch]
    The full proof is in \Cref{app:unigramdeletioncost}.
    Before deletion, each expected occurrence of $\subword$ contributes $-\!\log \ptheta(\subword)$ to the negative log-likelihood.
    Under the local replacement approximation, deleting $\subword$ replaces it with $\replacementseq_{\subword}$, which contributes $-\log \ptheta(\replacementseq_{\subword})$.
    As the number of such occurrences is $\countsnew{\subword}{\dataset}$, we get \cref{eq:unigram_deletion_cost}.
\end{proof}

After each deletion, we update our encoding function $\directtoken[\vocab_k]$ by re-estimating $\ptheta$.
We do this by updating $\btheta$ via an expectation--maximisation procedure, maximising the log-likelihood of $\dataset$.\footnote{The details of this procedure fall out of our work's scope; for an accessible derivation, see \citet{meister2026unigramlm}.} 
This updated $\ptheta$ defines the segmentations produced by $\directtoken$, which we use to re-segment our dataset.
We now present an analogous update for compression.

\begin{restatable}{lemma}{compmaxdeletioncost}
\label{lemma:compmaxdeletioncost}
Under a local replacement approximation, the estimated change in compression due to deleting a token $\subword \in \vocab \setminus \alphabet$ can be approximated as:\looseness=-1
\begin{align}
    \tokendelta{\subword}{\objectivefunccomp(\cdot,\dataset)}
    =
    \countsnew{\subword}{\dataset}
    \left(
    |\replacementseq_{\subword}| - 1
    \right),
    \label{eq:compmax_deletion_cost}
\end{align}
where $\replacementseq_{\subword}$ is the compression-optimal replacement of $\subword$ after deletion:
$ \replacementseq_{\subword}
    \defeq \directtoken[\vocab_{k-1}\setminus\{\subword\}](\subword) $.
\end{restatable}
\begin{proof}[Proof sketch]
    The full proof is in \Cref{app:compmaxdeletioncost}.
    After deleting $\subword$, the local replacement uses $|\replacementseq_{\subword}|$ tokens instead, thus increasing the corpus length by $|\replacementseq_{\subword}|-1$.
    There are $\countsnew{\subword}{\dataset}$ such occurrences.
\end{proof}

Updating this compression-based tokeniser after a deletion is easier than for log-likelihood. 
No parameters need to be estimated: the vocabulary $\vocab_k$ alone determines the encoding function, so we simply re-segment the corpus under $\directtoken[\vocab_k]$.

\newcommand{\pmi}{\texttt{PMI}\xspace}

\vspace{-4pt}
\section{Connection to Prior Work}
\label{sec:prior_work}

As mentioned above, the tokenisation learning algorithms we propose are not entirely without precedent, and related to prior methods like WordPiece and PathPiece. 
In this section, we discuss this connections in detail.

\vspace{-4pt}
\subsection{\GreedyLL{} and Likelihood-Guided Merging}
\label{sec:prior_greedyll}

As noted in \cref{sec:method_bottomup}, \GreedyLL{} favours pairs of tokens that are not only frequent, but which co-occur more often than chance.
In fact, let the pointwise mutual information between two tokens be:\looseness=-1
\begin{subequations}
\begin{align}
    \operatorname{PMI}_{\mathrm{adj}}(\subword^1,\subword^2)
    &\defeq
    \log
    \frac{
        \ptheta(\subword^1,\subword^2)
    }{
        \ptheta(\subword^1)\,\ptheta(\subword^2)
    } \\
    &=
    \log
    \frac{
        \totcnt \, \countsnew{\subword^1,\subword^2}{\dataset}
    }{
        \countsnew{\subword^1}{\dataset}
        \countsnew{\subword^2}{\dataset}
    }.
    \label{eq:adj_pmi}
\end{align}
\end{subequations}
Relying on the PMI to operationalise this notion of more-often-than-chance co-occurrence, we can use a Taylor approximation of $\tokendelta{\subword^1,\subword^2}{\objectivefuncll(\cdot,\dataset)}$ to make this explanation of \GreedyLL{} explicit.
\begin{restatable}{lemma}{greedyllpmi}
\label{lemma:greedyllpmi}
Under a first-order Taylor approximation, the change in log-likelihood due to merging a token-pair can be approximated as:
\begin{align}
    \tokendelta{\subword^1,\subword^2}{\objectivefuncll(\cdot,\dataset)}
    \approx
    \countsnew{\subword^1,\subword^2}{\dataset}
    \left(
    \operatorname{PMI}_{\mathrm{adj}}(\subword^1,\subword^2)
    - 1
    \right).
    \label{eq:greedyll_pmi_corrected}
\end{align}
\end{restatable}
\begin{proof}[Proof sketch]
    The full proof is in \Cref{app:greedyllpmi}.
    Starting from \cref{lemma:equivgreedyll}, apply a first-order Taylor expansion independently to each $x\log x$ term.
    Collecting the resulting terms returns the equation above.
\end{proof}%

\Cref{eq:greedyll_pmi_corrected} is closely related to the objective of another famous \greedyalg{} tokeniser, \defn{WordPiece}.
WordPiece's definition, however, is far from unified and has shifted over time.
As originally proposed by \citet{schuster-nakajima-2012-voice}, WordPiece selects the unit that most increases the likelihood of the data---i.e., using the same objective as \GreedyLL{}---but approximates the local objective, selecting multiple merges in parallel and only approximately updating their model $\ptheta$.
WordPiece was later popularised by \citet{wu2016google}, whose description suggests, in two consecutive sentences, optimising either a log-likelihood or a compression objective.%
\footnote{``The wordpiece model is generated using a data-driven approach to maximize the language-model likelihood
of the training data [...] Given a training corpus and a number of desired tokens $D$, the optimization problem is to select $D$ wordpieces such that the resulting corpus is minimal in the number of wordpieces when segmented according to the chosen wordpiece model.'' \citep{wu2016google}}
Accordingly, the widely used HuggingFace implementation of WordPiece optimises for compression \citep{wolf-etal-2020-transformers}.
Finally, WordPiece is often described
as merging the pair with the highest $\operatorname{PMI}_{\mathrm{adj}}$ \citep{schmidt-etal-2024-tokenization,lesci-etal-2025-causal,hf-llm-course-wordpiece}, which relates to \cref{eq:greedyll_pmi_corrected} but drops the $\countsnew{\subword^1,\subword^2}{\dataset}$ factor.
\GreedyLL{} is thus closest to the original WordPiece's description, but differs from all existing versions of this method.\looseness=-1

Importantly, even closer to \cref{eq:greedyll_pmi_corrected} is S-BPE \citep{vilar-federico-2021-statistical}, which scores merges by a count-weighted PMI criterion.
The two expressions are not identical, as S-BPE omits the $-1$ term, but are otherwise equivalent.
We experiment with a tokenisation learning algorithm based on \cref{lemma:greedyllpmi} in some of our experiments, which we then label as \GreedyLL $\approx$; when no such label is present, \GreedyLL refers to the method described by \cref{lemma:equivgreedyll}. 

\subsection{\CompMax{} and Top-Down Compression}
\label{sec:prior_compmax}

The closest prior method to \CompMax{} is PathPiece \citep{schmidt-etal-2024-tokenization}, which likewise prunes a large initial vocabulary under a compression objective.
The two differ in the deletion score: when scoring the removal of a token $\subword$, PathPiece also considers re-segmentation of other vocabulary tokens that contain $\subword$ as a substring, whereas \CompMax{} considers only $\subword$'s own local replacement.
Adopting the PathPiece variant would change the algorithm in more ways than just the objective, so we keep \CompMax{} as the direct compression counterpart of \UnigramLM{}, differing from it in objective alone.

\begin{table*}[t]
\centering
\adjustbox{max width=\textwidth}{%
\begin{tabular}{lllcccccccccc}
\toprule
Vocab\!\!\!\! & Tokeniser & Objective & Search & $\objectivefunccomp$ $\downarrow$ & $\objectivefuncll$ $\downarrow$ & Entropy & Zipf $\alpha$ & BPT $\uparrow$ & Tok Len & Vocab Util $\uparrow$ & Cov. 50\% \\
\midrule

\multirow{4}{*}{8k}
& \BPE{}        & $\objectivefunccomp$ & bottom-up & 59,559,835 & $6.259 \times 10^{8}$ & 10.508 & 1.033 & 3.844 & 5.53 & \textbf{99.9\%} & 242 \\
& \CompMax{}    & $\objectivefunccomp$ & top-down  & \textbf{59,304,113} & $6.259 \times 10^{8}$ & 10.554 & 0.919 & \textbf{3.861} & 5.67 & 99.3\% & 268 \\
& \GreedyLL{}   & $\objectivefuncll$   & bottom-up & 61,170,714 & $\bm{6.167 \times 10^{8}}$ & 10.081 & 1.315 & 3.743 & 6.05 & 99.4\% & 152 \\
& \UnigramLM{}  & $\objectivefuncll$   & top-down  & 68,444,350 & $6.221 \times 10^{8}$ &  9.089 & 1.093 & 3.345 & 6.72 & 99.3\% & 39 \\
\midrule

\multirow{4}{*}{32k}
& \BPE{}        & $\objectivefunccomp$ & bottom-up & \textbf{50,327,096} & $5.544 \times 10^{8}$ & 11.016 & 1.195 & \textbf{4.549} & 6.52 & \textbf{99.8\%} & 269 \\
& \CompMax{}    & $\objectivefunccomp$ & top-down  & 52,204,230 & $5.667 \times 10^{8}$ & 10.856 & 1.074 & 4.386 & 6.29 & \textbf{99.8\%} & 192 \\
& \GreedyLL{}   & $\objectivefuncll$   & bottom-up & 50,920,853 & $\bm{5.505 \times 10^{8}}$ & 10.811 & 1.378 & 4.496 & 6.98 & 99.0\% & 211 \\
& \UnigramLM{}  & $\objectivefuncll$   & top-down  & 58,161,017 & $5.697 \times 10^{8}$ &  9.796 & 1.272 & 3.937 & 7.03 & \textbf{99.8\%} & 51 \\
\midrule

\multirow{4}{*}{128k}
& \BPE{}        & $\objectivefunccomp$ & bottom-up & \textbf{46,859,565} & $5.216 \times 10^{8}$ & 11.132 & 1.484 & \textbf{4.886} & 7.06 & 98.9\% & 204 \\
& \CompMax{}    & $\objectivefunccomp$ & top-down  & 50,312,483 & $5.484 \times 10^{8}$ & 10.900 & 1.443 & 4.551 & 6.01 & \textbf{99.8\%} & 149 \\
& \GreedyLL{}   & $\objectivefuncll$   & bottom-up & 47,111,733 & $\bm{5.203 \times 10^{8}}$ & 11.044 & 1.574 & 4.860 & 7.48 & 96.8\% & 191 \\
& \UnigramLM{}  & $\objectivefuncll$   & top-down  & 55,546,423 & $5.552 \times 10^{8}$ &  9.996 & 1.717 & 4.122 & 6.46 & 97.1\% & 53 \\
\bottomrule
\end{tabular}}
\vspace{-3pt}
\caption{Intrinsic tokenisation results on a held-out test set (47{,}384 documents). Columns are defined in the Evaluation paragraph of \Cref{sec:setup}. \Cref{tab:app_distribution_eng_vs_multi} in \Cref{app:english_vs_multi} reports the multilingual counterpart of this table.}
\label{tab:tokenizer_stats_grouped}
\vspace{-10pt}
\end{table*}

\section{Experimental Setup}
\label{sec:setup}

We will now compare our four tokenisers: \BPE{}, \GreedyLL{}, \CompMax{}, and \UnigramLM{}.
All four are trained with an identical preprocessing pipeline (NFC normalisation followed by a byte-level pretokeniser using the GPT-2 regex) and on the same corpus, so that any difference between them stems only from the objective and the search procedure. \Cref{app:tokeniser_details} gives the exact configuration under which we train our tokenisers.
Full language model hyperparameter and architectural details are in \Cref{app:hyperparams}.

\paragraph{English setup.}
For English experiments, we train our tokenisers on a subset with approximately 2B tokens from FineWeb-Edu \citep{lozhkov2024fineweb-edu}, using vocabulary sizes of 8k, 32k, and 128k.
We then again use FineWeb-Edu to train language models using these tokenisers, evaluating 100M, 300M, 500M, and 1B parameter models.\footnote{Full hyperparameter and architectural details are in \Cref{app:hyperparams}.}
Results for 100M-parameter models are averaged over three random seeds, while larger models use one run per configuration.
Finally, unless otherwise stated, models are trained with a Chinchilla-style token budget of approximately $20\times$ as many training tokens as model parameters.\looseness=-1

\paragraph{Multilingual setup.}
Our multilingual corpus contains 20B tokens across five languages: 10B English tokens from FineWeb-Edu \citep{lozhkov2024fineweb-edu}, and 2.5B tokens each of German, Spanish, Turkish, and Chinese from FineWeb2 \citep{penedo2025fineweb2pipelinescale}.
We train our tokenisers on a 10\% subsample of this corpus, using a single vocabulary size of 128k.
We then train 1B-parameter language models on the full corpus.

\paragraph{Evaluation.}
We report both extrinsic and intrinsic metrics.
For \emph{extrinsic} metrics, we report our language models' bits per byte (\bpb), computed on a held-out validation set, 
and minimal-pair grammatical accuracy, computed on BLiMP \citep{warstadt2020blimp}, MultiBLiMP \citep{jumelet-etal-2026-multiblimp}, and ZhoBLiMP \citep{liu2025systematicassessmentlanguagemodels}.
For \emph{intrinsic} metrics, we report $\objectivefunccomp$ and $\objectivefuncll$. We also report the unigram entropy of our tokenised corpus, its bytes per token (BPT), average vocab token len (Tok Len), vocabulary utilisation (Vocab Util), Zipf $\alpha$, and Coverage 50\%. Detailed definitions are in \cref{app:metric_defs}.\looseness=-1

\vspace{-4pt}
\section{Results}
\label{sec:tokenizer_comparison}

\vspace{-4pt}
\subsection{Intrinsic Evaluation}

\paragraph{Compression vs.\ Log-likelihood.}
\cref{tab:tokenizer_stats_grouped} presents both the $\objectivefunccomp$ and $\objectivefuncll$ achieved by our tokenisers. 
From this table, we can see that, for a small vocabulary (of $8k$ tokens), the compression and log-likelihood scores achieved by the different tokenisers seem to align with their objective function: \BPE and \CompMax achieve better compression, while \UnigramLM and \GreedyLL achieve better log-likelihood.
For larger vocabularies, however, this surprisingly does not hold, with the bottom-up tokenisers consistently achieving both better compression and log-likelihood, irrespective of their objective function.
Another result of interest is with respect to a tokeniser's entropy, which, intuitively, measures how evenly a tokeniser spreads the frequency of its tokens.
This quantity has a strong relationship to our optimisation objectives, which can be seen when we rewrite it:
\begin{subequations}
\begin{align}
    \mathrm{Ent}(\tokenise, \dataset) &\defeq - \sum_{\subword \in \vocab} \ptheta(\subword) \log \ptheta(\subword) \\
    &= \frac{\objectivefuncll(\tokenise, \dataset)}{\objectivefunccomp(\tokenise, \dataset)}
\end{align}
\end{subequations}
This follows directly from the definitions: writing $\ptheta(\subword) = \countsnew{\subword}{\dataset}/\totcnt$, and noting that $\totcnt$, the total number of tokens, is exactly $\objectivefunccomp(\tokenise, \dataset)$, we get $-\sum_{\subword} \ptheta(\subword) \log \ptheta(\subword) = -\frac{1}{\totcnt} \sum_{\subword} \countsnew{\subword}{\dataset} \log \ptheta(\subword) = \objectivefuncll(\tokenise, \dataset) / \totcnt$.
Looking at the entropy values in \cref{tab:tokenizer_stats_grouped} we see a similar behaviour as before: for small vocabularies, entropy results align with the choice of tokenisation objective; for a larger vocabulary, however, top-down methods achieve consistently lower entropy.
Future work should further investigate this relationship between objective and search procedure.\looseness=-1

\begin{figure*}[t]
\centering
\includegraphics[width=\textwidth]{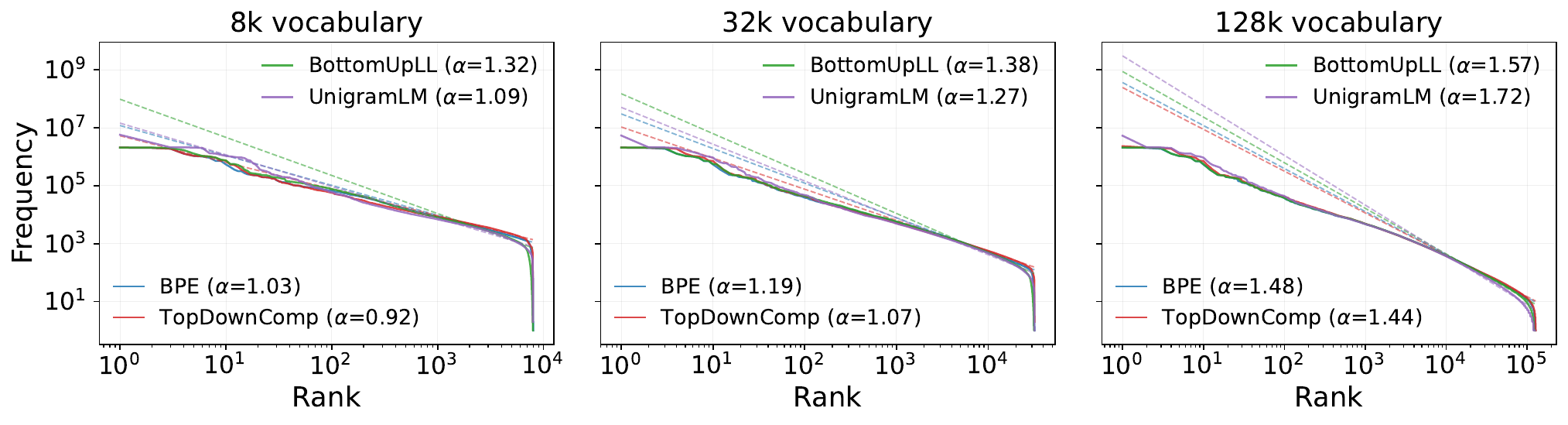}
\vspace{-25pt}
\caption{Frequency--rank token distributions induced by each tokeniser on the held-out test set. Solid lines are empirical curves; dashed lines are the fitted power laws, whose negated slopes are the Zipf $\alpha$ values in \Cref{tab:tokenizer_stats_grouped}.}
\label{fig:zipf}
\end{figure*}

\begin{figure}[t]
\centering
\includegraphics[width=\columnwidth]{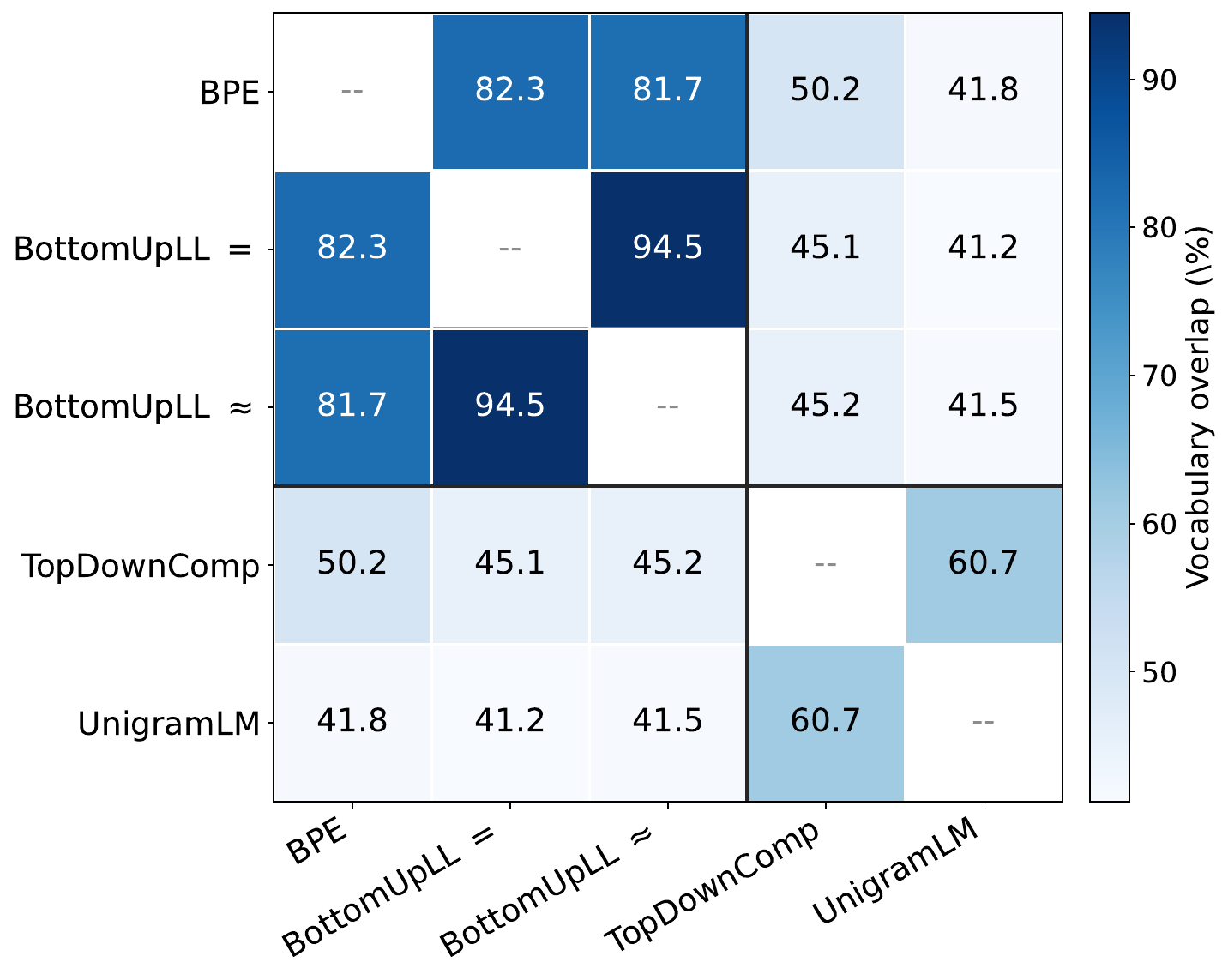}
\vspace{-15pt}
\caption{Pairwise vocabulary overlap at 128k vocabulary size. 
\GreedyLL{} $=$ denotes the exact variant of this method, which scores merges with \cref{lemma:equivgreedyll}, while \GreedyLL{} $\approx$ denotes the approximate variant, which uses the \pmi{}-based approximation of \cref{lemma:greedyllpmi}.
See \cref{tab:app_vocab_overlap_multi} in \Cref{app:english_vs_multi} for a multilingual counterpart of this table.}
\label{tab:vocab_overlap}
\vspace{-4pt}
\end{figure}

\paragraph{Distributional Properties of Tokens.}
\Cref{tab:tokenizer_stats_grouped} also reports our tokenisers' Zipf $\alpha$ and Coverage 50\%, where
Zipf $\alpha$ is the negated slope of a frequency vs.\ rank log-log curve,
and Coverage 50\% is the smallest $\ell$ such that the $\ell$ most frequent tokens together account for at least half of all token occurrences in our corpus.
A larger Zipf $\alpha$ thus means that token frequency decreases more rapidly with rank, while lower Coverage 50\% values indicate that token occurrences are concentrated among fewer token types.
\Cref{fig:zipf} shows corresponding frequency--rank curves.
Within each search family, the likelihood-based method induces a more concentrated token-frequency distribution than the compression-based method: \GreedyLL{} has higher Zipf $\alpha$, lower entropy, and lower Coverage 50\% than \BPE{}, and \UnigramLM{} shows the same pattern relative to \CompMax{}.
Thus, while the search procedure may dictate the values of $\objectivefunccomp$ and $\objectivefuncll$, the used objective is still visible in a tokenised corpus' empirical token-frequency distribution.

\newcommand{\std}[1]{{\footnotesize $\pm$ #1}}

\paragraph{Other Intrinsic Metrics.}
Finally, \cref{tab:tokenizer_stats_grouped} also presents our tokenisers' vocabulary utilisation---i.e., the proportion of vocabulary items used at least once on a test set of 47{,}384 documents---as well as the average length of the tokens in their vocabularies.
Vocabulary utilisation is close to 100\% at 8k and 32k, but at 128k, likelihood-based tokenisers have more unused tokens; \Cref{app:merge_dynamics} traces this to the intermediate tokens that \GreedyLL{} creates during merging and then never uses again.
Token length behaves differently across vocabulary sizes:
at small vocabulary sizes, \pruningalg methods present longer tokens, as they can retain these tokens from the initial seed vocabulary;
at larger vocabulary sizes, though, \greedyalg methods present longer tokens instead, as they can build these tokens from scratch through successive merges.
Interestingly, across both search procedures, the likelihood-based tokenisers have longer tokens than their compression-based counterparts.

\paragraph{Vocabulary Composition.}
We now move on to analysing the composition of our tokenisers' vocabularies.
To this end, we compute their similarity, defined as $\nicefrac{|\vocab_A \cap \vocab_B|}{|\vocab_A|}$.
\Cref{tab:vocab_overlap} presents these results.
From this table, we see that learned vocabularies cluster primarily by search procedure: e.g., \greedyalg methods share substantially more tokens with each other than with \pruningalg methods.
This indicates that a tokeniser's search procedure has a large effect on which tokens are included in its final vocabulary.
\Cref{fig:venn5} in \Cref{app:vocab_venn} shows the full set-overlap structure of all five vocabularies.
Furthermore, per-example qualitative segmentations for the four tokenisers are shown in \Cref{app:qual_ex}.

\vspace{-4pt}
\subsection{Extrinsic Evaluation}

\begin{table*}[t]
\centering
\adjustbox{max width=\textwidth}{%
\begin{tabular}{lllcc cc cc cc cc}
\toprule
& & & & \multicolumn{2}{c}{100M} & \multicolumn{2}{c}{300M} & \multicolumn{2}{c}{500M} & \multicolumn{2}{c}{1B} \\
\cmidrule(lr){5-6}\cmidrule(lr){7-8}\cmidrule(lr){9-10}\cmidrule(lr){11-12}
Vocab\!\!\!\! & Tokeniser & Objective & Search & BLiMP $\uparrow$ & \bpb $\downarrow$ & BLiMP $\uparrow$ & \bpb $\downarrow$ & BLiMP $\uparrow$ & \bpb $\downarrow$ & BLiMP $\uparrow$ & \bpb $\downarrow$ \\
\midrule
\multirow{4}{*}{8k}
& \BPE{}        & $\objectivefunccomp$ & bottom-up & .686 \std{6e-3} & 1.0736 \std{2e-4} & .7707 & .8739 & .7744 & .8435 & -- & -- \\
& \CompMax{}    & $\objectivefunccomp$ & top-down  & .687 \std{4e-3} & 1.0750 \std{8e-4} & .7699 & .8758 & .7938 & .8441 & -- & -- \\
& \GreedyLL{}   & $\objectivefuncll$   & bottom-up & \textbf{.703 \std{8e-3}} & \textbf{1.0715 \std{6e-4}} & \textbf{.7841} & \textbf{.8727} & \textbf{.7951} & \textbf{.8420} & -- & -- \\
& \UnigramLM{}  & $\objectivefuncll$   & top-down  & .698 \std{5e-3} & 1.0819 \std{4e-4} & .7769 & .8758 & .7676 & .8472 & -- & -- \\
\midrule
\multirow{4}{*}{32k}
& \BPE{}        & $\objectivefunccomp$ & bottom-up & .714 \std{8e-3} & 1.0392 \std{5e-4} & \textbf{.7904} & .8584 & .7958 & .8290 & -- & -- \\
& \CompMax{}    & $\objectivefunccomp$ & top-down  & \textbf{.727 \std{2e-2}} & 1.0394 \std{3e-4} & .7895 & .8581 & .7986 & .8307 & -- & -- \\
& \GreedyLL{}   & $\objectivefuncll$   & bottom-up & .717 \std{7e-3} & \textbf{1.0374 \std{1e-4}} & .7827 & \textbf{.8564} & .7899 & \textbf{.8281} & -- & -- \\
& \UnigramLM{}  & $\objectivefuncll$   & top-down  & .719 \std{8e-3} & 1.0441 \std{9e-4} & .7776 & .8610 & \textbf{.8015} & .8337 & -- & -- \\
\midrule
\multirow{4}{*}{128k}
& \BPE{}        & $\objectivefunccomp$ & bottom-up & \textbf{.740 \std{4e-3}} & \textbf{1.0138 \std{4e-4}} & .797 & \textbf{.8462} & .813 & .8210 & .805 & .7790 \std{4e-4} \\
& \CompMax{}    & $\objectivefunccomp$ & top-down  & \textbf{.740 \std{8e-3}} & 1.0172 \std{4e-4} & .797 & .8492 & \textbf{.815} & .8226 & .816 & .7816 \std{3e-4} \\
& \GreedyLL{}   & $\objectivefuncll$   & bottom-up & .725 \std{6e-3} & 1.0149 \std{9e-4} & .792 & .8467 & .798 & \textbf{.8207} & .799 & \textbf{.7803 \std{26e-4}} \\
& \UnigramLM{}  & $\objectivefuncll$   & top-down  & .738 \std{7e-3} & 1.0298 \std{8e-4} & \textbf{.799} & .8547 & .803 & .8283 & \textbf{.818} & .7872 \std{1e-4} \\
\bottomrule
\end{tabular}}
\caption{Extrinsic tokenisation results for language models trained in English. Cells marked ``--'' were not run; 1B models were trained only at 128k vocabulary. 100M results are averaged over three random seeds; for the 1B models, \bpb is likewise averaged over three seeds (42, 43, 44), while BLiMP uses a single seed. Values following $\pm$ are standard deviations across seeds. Lower is better for \bpb; higher is better for BLiMP.}
\label{tab:main_results}
\end{table*}

\begin{table*}[t]
\centering
\adjustbox{max width=\textwidth}{%
\begin{tabular}{lllcccccccccccc}
\toprule
&&&& \multicolumn{5}{c}{Minimal-pair accuracy $\uparrow$} & \multicolumn{5}{c}{\bpb $\downarrow$} \\
\cmidrule(lr){5-9}\cmidrule(lr){10-14}
Tokeniser & Objective & Search & BLiMP $\uparrow$ & eng & deu & spa & tur & cmn & eng & deu & spa & tur & cmn \\
\midrule
\BPE{}        & $\objectivefunccomp$ & bottom-up & \textbf{.812} & \textbf{.974} & .955 & .958 & .890 & .799 & \textbf{.8080} & .9426 & \textbf{.9025} & .8936 & 1.2464 \\
\CompMax{}    & $\objectivefunccomp$ & top-down  & .808 & .969 & \textbf{.966} & .956 & .853 & .718 & .8276 & .9711 & .9284 & .9260 & 1.3430 \\
\GreedyLL{}   & $\objectivefuncll$   & bottom-up & .808 & .970 & .952 & .950 & .873 & \textbf{.803} & .8082 & \textbf{.9416} & .9030 & \textbf{.8929} & \textbf{1.2436} \\
\UnigramLM{}  & $\objectivefuncll$   & top-down  & .803 & .973 & .964 & \textbf{.959} & \textbf{.897} & .728 & .8311 & .9683 & .9303 & .9252 & 1.3575 \\
\bottomrule
\end{tabular}}
\caption{Extrinsic tokenisation results for 1B parameter models trained in a multilingual setting. Minimal-pair accuracy is evaluated on BLiMP \citep{warstadt2020blimp}, MultiBLiMP \citep[eng, deu, spa, tur;][]{jumelet-etal-2026-multiblimp}, and ZhoBLiMP \citep[cmn;][]{liu2025systematicassessmentlanguagemodels}. Lower is better for \bpb; higher is better for minimal-pair accuracy.}
\label{tab:main_results_multi}
\end{table*}

\paragraph{\bpb.}
\Cref{tab:main_results,tab:main_results_multi} present extrinsic results for both the English and multilingual settings.
From these tables, we can see that \greedyalg tokenisers consistently achieve lower \bpb scores than \pruningalg.
This holds true for all our experimental conditions, with the only exception of the English 300M models at 32k vocabulary, where 
\CompMax{} outperforms \BPE{} by a negligible margin ($0.0003$ \bpb); the likelihood-based pair (\GreedyLL{} vs.\ \UnigramLM{}) favours \greedyalg{} in every setting.\footnote{Because several of these differences are small, we additionally verify their significance with a paired document-level bootstrap over the evaluation set, which confirms the statistical significance of each of these comparisons (see details in \Cref{app:significance}).}
Comparing tokenisers within a single search procedure, we see that the choice of objective is still relevant, albeit less strongly.
Compression-based methods tend to outperform log-likelihood; but there are several exceptions where \GreedyLL outperforms all other methods.\looseness=-1

\paragraph{Minimal-pair accuracy.}
\Cref{tab:main_results,tab:main_results_multi} also present scores on a minimal-pair grammaticality task, for both the English and multilingual settings.
Notably, these results do not show the same clear preference as \bpb for \greedyalg over \pruningalg.
In fact, in the English setting, \pruningalg methods often outperform the \greedyalg ones, with \UnigramLM outperforming \GreedyLL and \CompMax outperforming \BPE for the 1B models.
In the multilingual setting, no consistent pattern arises either, and all tokenisers perform best in at least one language.
Our results in German, Spanish, and Turkish, however, suggest that \pruningalg methods may be helpful for morphologically rich languages.
More broadly, our multilingual minimal-pair accuracy results suggest different tokenisers induce different inductive biases, which will in turn be uniquely suited to different languages; we leave this direction to future work.

\section{Conclusion}

Our paper disentangles two important design choices in tokeniser learning algorithms: the objective used to score a tokeniser and the search procedure used to optimise it.
We analyse two objectives (compression vs.\ log-likelihood) and search procedures (\greedyalg vs.\ \pruningalg), introducing two new tokenisers in the process: \GreedyLL{} and \CompMax{}.
In nearly all experimental conditions, language models trained with \greedyalg tokenisers outperformed \pruningalg ones in terms of bits-per-byte.
Our intrinsic evaluations also highlighted a surprising pattern: \greedyalg tokenisers produced both more compressed and higher log-likelihood token-strings than \pruningalg tokenisers, independent of the used objective function.
The used objective still affects the resulting tokeniser, though, with log-likelihood objective leaving a clear impact on the empirical rank--frequency distribution of its tokens.
Overall, our results show that a tokeniser's effect on language modelling cannot be explained by the objective alone, and that the used search procedure is an important component of tokeniser design.
We hope future work will explore the impact of other optimisation objectives and search procedures\footnote{Investigating the search procedures of, e.g.,\citet{tempus2026tokenisationconvexrelaxations} or \citet{chizhov-etal-2024-bpe}.} in language modelling.

\section*{Limitations}

We list a few limitations with our study here.

\paragraph{Scale.}
First, we only train language models up to 1B parameters, and  results could shift for larger models or longer training horizons.

\paragraph{Languages.}
Second, our multilingual experiments cover five selected languages: English, German, Spanish, Turkish, and Chinese.
Four of these use the Latin script and all are relatively high-resource, so our multilingual results should be read as applying to these five languages rather than to multilingual tokenisation in general.
How our findings transfer to low-resource languages, and to scripts with markedly different orthographic conventions, remains an open question that we leave to future work.

\paragraph{Single seed at scale.}
Third, for computational reasons, the 300M and 500M models are trained with a single seed per configuration, as are the multilingual 1B models; the 100M and English 1B models are averaged over three seeds; small differences in results across tokenisers should thus be interpreted with caution.

\paragraph{Tokeniser-training corpus.}
Fourth, in all our experiments, each tokeniser is trained on the same corpus as the language model trained on top of it. We thus do not analyse the impact of training models on data that differ in domain or style from its tokeniser.

\paragraph{Approximate top-down deletion scores.}
Fifth, as discussed in \cref{sec:method_topdown}, \pruningalg{} deletion costs rely on the \emph{local replacement approximation}, whereas \greedyalg{} merge gains are exact for the current step.
We asses the impact of this approximation by, at a small scale, training tokenisers for which we compute exact deletion scoring.
We find that the local-replacement-approximation tokenisers present an $81.5\%$ vocabulary overlap with these exact ones. 
Whether this holds at the vocabulary sizes used in our main experiments, however, remains open.
A second approximation we rely on for \pruningalg methods is that we \emph{prune multiple tokens} per round rather than one at a time.
We assess the impact of this approximation by re-training the \pruningalg{} tokenisers with pruning rates of $1\%$ and $0.1\%$ of the vocabulary per round, instead of the $10\%$ used in our main experiments: the resulting vocabularies overlap with the $10\%$ ones by over $97\%$ for \UnigramLM{} and over $99\%$ for \CompMax{}, suggesting that this approximation has a limited effect on the learned vocabulary.
Details for both experiments are in \Cref{app:prune_rate}.

\paragraph{Evaluation.}
Finally, we evaluate models with BPB and minimal-pair grammatical benchmarks (BLiMP, MultiBLiMP, ZhoBLiMP). We do not measure other downstream task performance, reasoning ability, or long-context behaviour; how our models perform in those different settings is thus left open.

\section*{Acknowledgements}

We thank Philip Whittington for helpful discussions and feedback throughout this project, and the anonymous reviewers for their constructive comments. This work was supported as part of the ``Swiss AI initiative'' by a grant from the Swiss National Supercomputing Centre (CSCS) under project ID a0229 on Alps.

\bibliography{custom}

\newpage
\appendix

\section{Proof of \texorpdfstring{{\Cref{lemma:equivbpefreq}}}{Lemma}}
\label{app:equivbpefreq}

\equivbpefreq*
\begin{proof}
    To prove this is the case, first note that replacing a token-pair $\merge = \langle \subword^1, \subword^2 \rangle$ with merged token $\subword^1 \circ \subword^2$ saves exactly 1 token.
    As the number of non-overlapping occurrences of a token-pair is also the number of possible merges, it is equivalent to the compression this token-pair would achieve:
    \begin{align}
        &\objectivefunc(\bottomuptoken[\merges_{<k}], \dataset)  - \objectivefunc(\bottomuptoken[\merges_{<k} \circ \langle \subword^1, \subword^2 \rangle], \dataset) \nonumber\\
        &\qquad\qquad\qquad= \countsnew{\subword^1, \subword^2}{\dataset}.
    \end{align}
    We can now trivially conclude the proof:
\begin{subequations}
    \begin{align}
        \mergegood
        &= \argmin_{\merge \in \alphabet^+ \times \alphabet^+}
            \objectivefunc(\bottomuptoken[\merges_{<k} \circ \merge], \dataset) \\
        &= \argmax_{\merge \in \alphabet^+ \times \alphabet^+}
            -\objectivefunc(\bottomuptoken[\merges_{<k} \circ \merge], \dataset) \\
        &= \argmax_{\merge \in \alphabet^+ \times \alphabet^+}
            \objectivefunc(\bottomuptoken[\merges_{<k}], \dataset) \\
        &\qquad\qquad\qquad\quad - \objectivefunc(\bottomuptoken[\merges_{<k} \circ \merge], \dataset) \nonumber \\
        &= \argmax_{\langle \subword^1, \subword^2 \rangle \in \alphabet^+ \times \alphabet^+}
            \countsnew{\subword^1, \subword^2}{\dataset}.
    \end{align}
\end{subequations}
This concludes the proof.
\end{proof}

\section{Updating Counts in \greedyalg Procedures}
\label{app:update_counts}

Both \BPE{} and \GreedyLL{} repeatedly select a token-pair based on its merge gain on the current tokenised corpus.
A naive implementation would recompute token and token-pair counts after each merge, which is expensive on large corpora.
Instead, our implementation maintains these counts incrementally and updates only the parts of the corpus affected by the selected merge.

The corpus is represented as a collection of unique word types, each with an associated frequency.
Each word type stores its current tokenisation using a linked token structure: every token stores pointers to its predecessor and successor, allowing merges to be applied locally without reconstructing the full token string.
We maintain a mapping
\begin{align}
    & \mathrm{pair\_to\_words}:
    \langle \subword^1,\subword^2\rangle  \\
    & \qquad \mapsto
    \{
    \text{word types containing }
    \langle \subword^1,\subword^2\rangle
    \}, \nonumber
\end{align}
which maps each token-pair to the word types in which that pair currently occurs.
We also maintain a reverse index from tokens to the token-pairs containing them, which allows us to identify which candidate scores may be affected after a merge.

When a merge $\merge=\langle \subword^1,\subword^2\rangle$ is selected, write
\[
    \subword^{\mathrm{new}}
    \defeq
    \subword^1 \circ \subword^2 .
\]
Each affected word type is scanned through its linked token sequence, and non-overlapping occurrences of $\langle \subword^1,\subword^2\rangle$ are replaced by $\subword^{\mathrm{new}}$.
The non-overlap condition is enforced during this scan: once an occurrence has been merged, the scan skips over the newly created token, so overlapping occurrences of the same pair are not counted twice.

Locally, an occurrence in context
\begin{align}
    \langle
    \subword^\ell,
    \subword^1,
    \subword^2,
    \subword^r
    \rangle
\end{align}
is replaced by
\begin{align}
    \langle
    \subword^\ell,
    \subword^{\mathrm{new}},
    \subword^r
    \rangle,
\end{align}
where $\subword^\ell$ and $\subword^r$ denote the immediate left and right neighbours, if they exist.
Therefore, only token-pairs in this local neighbourhood can change.
The old token-pairs
\begin{align}
    \langle \subword^\ell,\subword^1\rangle,
    \qquad
    \langle \subword^1,\subword^2\rangle,
    \qquad
    \langle \subword^2,\subword^r\rangle
\end{align}
are removed, and the new token-pairs
\begin{align}
    \langle \subword^\ell,\subword^{\mathrm{new}}\rangle,
    \qquad
    \langle \subword^{\mathrm{new}},\subword^r\rangle
\end{align}
are added.
Boundary cases are handled by omitting updates involving missing neighbours.

For one merged occurrence, the affected pair counts are updated as
\begin{subequations}
\begin{align}
    \countsnew{\subword^\ell,\subword^1}{\dataset}
    &\leftarrow
    \countsnew{\subword^\ell,\subword^1}{\dataset} - 1,
    \\
    \countsnew{\subword^1,\subword^2}{\dataset}
    &\leftarrow
    \countsnew{\subword^1,\subword^2}{\dataset} - 1,
    \\
    \countsnew{\subword^2,\subword^r}{\dataset}
    &\leftarrow
    \countsnew{\subword^2,\subword^r}{\dataset} - 1,
    \\
    \countsnew{\subword^\ell,\subword^{\mathrm{new}}}{\dataset}
    &\leftarrow
    \countsnew{\subword^\ell,\subword^{\mathrm{new}}}{\dataset} + 1,
    \\
    \countsnew{\subword^{\mathrm{new}},\subword^r}{\dataset}
    &\leftarrow
    \countsnew{\subword^{\mathrm{new}},\subword^r}{\dataset} + 1.
\end{align}
\end{subequations}
When word types have frequencies, these increments and decrements are weighted by the corresponding word-type frequency.

For \GreedyLL{}, we also update unigram counts, since the merge score in \cref{lemma:equivgreedyll} depends on both token counts and pair counts.
If $\countsnew{\subword^1,\subword^2}{\dataset}$ denotes the frequency-weighted number of non-overlapping occurrences of the selected pair that are actually merged, then the unigram counts are updated as
\begin{subequations}
\begin{align}
    \countsnew{\subword^1}{\dataset}
    &\leftarrow
    \countsnew{\subword^1}{\dataset}
    -
    \countsnew{\subword^1,\subword^2}{\dataset},
    \\
    \countsnew{\subword^2}{\dataset}
    &\leftarrow
    \countsnew{\subword^2}{\dataset}
    -
    \countsnew{\subword^1,\subword^2}{\dataset},
    \\
    \countsnew{\subword^{\mathrm{new}}}{\dataset}
    &\leftarrow
    \countsnew{\subword^1,\subword^2}{\dataset},
    \\
    \totcnt
    &\leftarrow
    \totcnt
    -
    \countsnew{\subword^1,\subword^2}{\dataset}.
\end{align}
\end{subequations}
For \BPE{}, the affected candidate scores are simply the updated pair counts, as in \cref{lemma:equivbpefreq}.
For \GreedyLL{}, affected candidate scores are then recomputed using the local likelihood-change expression in \cref{lemma:equivgreedyll}.

The priority queue of candidate token-pairs is maintained lazily.
After a merge, the selected pair is removed from the pair-to-word index, affected reverse-index entries are updated, and affected token-pairs are pushed back onto the heap with fresh scores.
Old heap entries are not removed immediately.
Instead, stale entries are rejected when popped: the implementation checks the current count and score, and discards the heap entry if it no longer matches the current state.
The heap is also periodically rebuilt to control the accumulation of stale entries.\looseness=-1

This gives an incremental implementation in which the cost of a merge is proportional to the number of affected word types and local token updates, plus the cost of heap maintenance.
In practice, this is much cheaper than recomputing all token and token-pair counts over the full corpus after every merge.

\section{Proof of \texorpdfstring{{\Cref{lemma:equivgreedyll}}}{Lemma}}
\label{app:greedyllgain}

{\renewcommand{\footnote}[1]{}
\equivgreedyll*}

\begin{proof}
We prove the result by explicitly writing the corpus log-likelihood before and after applying the merge.
Let the current tokenised corpus contain $\totcnt$ tokens in total.
For each token $\tilde{\subword}$, let $\countsnew{\tilde{\subword}}{\dataset}$ denote its corpus count.
Under the unigram token model, the empirical probability of $\tilde{\subword}$ is
\begin{align}
    \ptheta(\tilde{\subword})
    =
    \frac{\countsnew{\tilde{\subword}}{\dataset}}{\totcnt}.
\end{align}
Therefore, the corpus log-likelihood before applying the merge is
\begin{subequations}
\begin{align}
    \mathcal{L}_{\mathrm{before}}
    &=
    \sum_{\tilde{\subword}}
    \countsnew{\tilde{\subword}}{\dataset}
    \log \ptheta(\tilde{\subword}) \\
    &=
    \sum_{\tilde{\subword}}
    \countsnew{\tilde{\subword}}{\dataset}
    \log
    \frac{\countsnew{\tilde{\subword}}{\dataset}}{\totcnt} \\
    &=
    \sum_{\tilde{\subword}}
    \countsnew{\tilde{\subword}}{\dataset}
    \log \countsnew{\tilde{\subword}}{\dataset}
    -
    \totcnt \log \totcnt.
    \label{eq:greedyll_before}
\end{align}
\end{subequations}

Now consider merging the token-pair
$\langle \subword^1,\subword^2\rangle$
into the new token
$\subword^1 \circ \subword^2$.
For readability, write
\begin{align}
    \kappa
    \defeq
    \countsnew{\subword^1,\subword^2}{\dataset},
\end{align}
Assuming $\subword^1 \neq \subword^2$, applying this merge changes only the counts of $\subword^1$, $\subword^2$, and the new token $\subword^1 \circ \subword^2$.
Specifically,
\begin{subequations}
\begin{align}
    \countsnew{\subword^1}{\dataset}
    &\mapsto
    \countsnew{\subword^1}{\dataset} - \kappa,
    \\
    \countsnew{\subword^2}{\dataset}
    &\mapsto
    \countsnew{\subword^2}{\dataset} - \kappa,
    \\
    \countsnew{\subword^1 \circ \subword^2}{\dataset}
    &\mapsto
    \kappa.
\end{align}
\end{subequations}
All other token counts remain unchanged.
Since each merged occurrence replaces two tokens by one token, the total token count changes from $\totcnt$ to $\totcnt-\kappa$.
The log-likelihood after the merge is therefore
\begin{align}
    \mathcal{L}_{\mathrm{after}}
    &=
    \sum_{\tilde{\subword}
    \notin
    \{\subword^1,\subword^2,\subword^1\circ\subword^2\}}
    \countsnew{\tilde{\subword}}{\dataset}
    \log \countsnew{\tilde{\subword}}{\dataset}
    \\
    &\quad
    +
    \bigl(
    \countsnew{\subword^1}{\dataset}
    -
    \kappa
    \bigr)
    \log
    \bigl(
    \countsnew{\subword^1}{\dataset}
    -
    \kappa
    \bigr)
    \nonumber\\
    &\quad
    +
    \bigl(
    \countsnew{\subword^2}{\dataset}
    -
    \kappa
    \bigr)
    \log
    \bigl(
    \countsnew{\subword^2}{\dataset}
    -
    \kappa
    \bigr)
    \nonumber\\
    &\quad
    +
    \kappa \log \kappa
    -
    (\totcnt-\kappa)\log(\totcnt-\kappa). \nonumber
    \label{eq:greedyll_after}
\end{align}
and the log-likelihood change due to the merge is
\begin{align}
    \tokendelta{\subword^1,\subword^2}{\objectivefuncll(\cdot,\dataset)}
    =
    \mathcal{L}_{\mathrm{after}}
    -
    \mathcal{L}_{\mathrm{before}}.
\end{align}
Substituting \cref{eq:greedyll_before,eq:greedyll_after}, all unchanged token-count terms cancel, leaving
\begin{align} \label{eq:greedyll_exact}
    \tokendelta{\subword^1,\subword^2}{\objectivefuncll(\cdot,\dataset)}
    &=
    \bigl(
    \countsnew{\subword^1}{\dataset}
    -
    \kappa
    \bigr)
    \log
    \bigl(
    \countsnew{\subword^1}{\dataset}
    -
    \kappa
    \bigr)
    \\
    &\quad
    -
    \countsnew{\subword^1}{\dataset}
    \log
    \countsnew{\subword^1}{\dataset}
    \nonumber\\
    &\quad
    +
    \bigl(
    \countsnew{\subword^2}{\dataset}
    -
    \kappa
    \bigr)
    \log
    \bigl(
    \countsnew{\subword^2}{\dataset}
    -
    \kappa
    \bigr)
    \nonumber\\
    &\quad
    -
    \countsnew{\subword^2}{\dataset}
    \log
    \countsnew{\subword^2}{\dataset}
    \nonumber\\
    &\quad
    +
    \kappa \log \kappa
    \nonumber\\
    &\quad
    -
    (\totcnt-\kappa)\log(\totcnt-\kappa)
    +
    \totcnt\log\totcnt. \nonumber
\end{align}
which completes the proof.
\end{proof}

\section{Proof of \texorpdfstring{{\Cref{lemma:unigramdeletioncost}}}{Lemma}}
\label{app:unigramdeletioncost}

{\renewcommand{\footnote}[1]{}
\unigramdeletioncost*}

\begin{proof}
We prove the result under the local replacement approximation stated in the lemma.
That is, when scoring the deletion of $\subword$, we keep the current token probabilities $\ptheta$ fixed, keep all other segmentation decisions fixed, and approximate the effect of deletion by replacing each expected occurrence of $\subword$ with a replacement segmentation under $\vocab \setminus \{\subword\}$.

For \UnigramLM{}, token usage is measured by expected counts under the current posterior over valid segmentations. Thus, $\countsnew{\subword}{\dataset}$ is the expected number of times $\subword$ is used in the corpus. Before deleting $\subword$, each such expected occurrence contributes
\begin{align}
    -\log \ptheta(\subword)
\end{align}
to the negative log-likelihood objective.

After deleting $\subword$, the token can no longer be used as a single token. Under the local replacement approximation, each occurrence of $\subword$ is replaced by the highest-probability valid segmentation of the same string under the reduced vocabulary:
\begin{align}
    \replacementseq_{\subword}
    \defeq
    \argmax_{\substack{
        \tilde{\subwords} \in (\vocab \setminus \{\subword\})^* \\
        \subword \stringequiv \tilde{\subwords}
    }}
    \ptheta(\tilde{\subwords}).
\end{align}
Since $\ptheta$ is a unigram model, the probability of this replacement token-string is
\begin{align}
    \ptheta(\replacementseq_{\subword})
    =
    \prod_{\tilde{\subword} \in \replacementseq_{\subword}}
    \ptheta(\tilde{\subword}).
\end{align}
After replacement, each expected occurrence therefore contributes
\begin{align}
    -\log \ptheta(\replacementseq_{\subword})
\end{align}
to the negative log-likelihood objective.

Hence, the estimated increase in negative log-likelihood for one expected occurrence of $\subword$ is
\begin{align}
    & -\log \ptheta(\replacementseq_{\subword})
    - \bigl(-\log \ptheta(\subword)\bigr) \nonumber \\
    & \qquad =
    \log \ptheta(\subword)
    - \log \ptheta(\replacementseq_{\subword}).
\end{align}
Multiplying this per-occurrence increase by the expected number of occurrences $\countsnew{\subword}{\dataset}$ gives
\begin{align}
    \tokendelta{\subword}{\objectivefuncll(\cdot,\dataset)}
    \approx
    \countsnew{\subword}{\dataset}
    \left(
    \log \ptheta(\subword)
    -
    \log \ptheta(\replacementseq_{\subword})
    \right).
\end{align}
This proves the stated local approximation.
\end{proof}

\section{Proof of \texorpdfstring{{\Cref{lemma:compmaxdeletioncost}}}{Lemma}}
\label{app:compmaxdeletioncost}

\compmaxdeletioncost*

\begin{proof}
Let $\vocab$ be the current vocabulary, and suppose the corpus has already been segmented using the \CompMax{} encoding rule:
\begin{align}
    \directtoken[\vocab](\characters)
    =
    \argmin_{\substack{
        \subwords \in \vocab^* \\
        \characters \stringequiv \subwords
    }}
    |\subwords|.
\end{align}

Now consider deleting a token $\subword \in \vocab \setminus \alphabet$.
Before deletion, each current occurrence of $\subword$ contributes exactly one token to the corpus token count.
Therefore, the total contribution of all current occurrences of $\subword$ before deletion is
\begin{align}
    C_{\mathrm{before}}(\subword)
    =
    \countsnew{\subword}{\dataset}.
\end{align}

After deleting $\subword$, the token $\subword$ can no longer appear as a single token.
Under the local replacement calculation, each current occurrence of $\subword$ is independently replaced by its shortest decomposition under the reduced vocabulary $\vocab \setminus \{\subword\}$, denoted $\replacementseq_{\subword}$:
\begin{align}
    \replacementseq_{\subword}
    =
    \argmin_{\substack{
        \subwords' \in (\vocab \setminus \{\subword\})^* \\
        \subword \stringequiv \subwords'
    }}
    |\subwords'|.
\end{align}
Thus, under this local replacement calculation, each previous occurrence of $\subword$ contributes $|\replacementseq_{\subword}|$ tokens instead of one token.
The total contribution of these occurrences after deletion is
\begin{align}
    C_{\mathrm{after}}(\subword)
    =
    \countsnew{\subword}{\dataset}
    |\replacementseq_{\subword}|.
\end{align}

Therefore, the estimated increase in corpus length caused by locally replacing all current occurrences of $\subword$ is
\begin{subequations}
\begin{align}
    \tokendelta{\subword}{\objectivefunccomp(\cdot,\dataset)}
    &=
    C_{\mathrm{after}}(\subword)
    -
    C_{\mathrm{before}}(\subword) \\
    &=
    \countsnew{\subword}{\dataset}
    |\replacementseq_{\subword}|
    -
    \countsnew{\subword}{\dataset} \\
    &=
    \countsnew{\subword}{\dataset}
    \left(
    |\replacementseq_{\subword}| - 1
    \right).
\end{align}
\end{subequations}
which completes this proof.
\end{proof}

\section{Proof of \texorpdfstring{{\Cref{lemma:greedyllpmi}}}{Lemma}}
\label{app:greedyllpmi}

\greedyllpmi*

\begin{proof}
As in \cref{app:greedyllgain}, we write
$\kappa \defeq \countsnew{\subword^1,\subword^2}{\dataset}$.
Starting from \cref{eq:greedyll_exact}, we have
\begin{align}
    &\tokendelta{\subword^1,\subword^2}{\objectivefuncll(\cdot,\dataset)} = 
    \label{eq:pmi_exact_start} \\
    &\quad
    \bigl(
    \countsnew{\subword^1}{\dataset}
    -
    \kappa
    \bigr)
    \log
    \bigl(
    \countsnew{\subword^1}{\dataset}
    -
    \kappa
    \bigr)
    -
    \countsnew{\subword^1}{\dataset}
    \log
    \countsnew{\subword^1}{\dataset}
    \nonumber\\
    &\quad
    +
    \bigl(
    \countsnew{\subword^2}{\dataset}
    -
    \kappa
    \bigr)
    \log
    \bigl(
    \countsnew{\subword^2}{\dataset}
    -
    \kappa
    \bigr)
    -
    \countsnew{\subword^2}{\dataset}
    \log
    \countsnew{\subword^2}{\dataset}
    \nonumber\\
    &\quad
    +
    \kappa \log \kappa
    \nonumber\\
    &\quad
    -
    (\totcnt-\kappa)\log(\totcnt-\kappa)
    +
    \totcnt\log\totcnt. \nonumber
\end{align}

Now, define
    $f(a) = a\log a$,
which implies
    $f'(a) = \log a + 1$.
A first-order Taylor expansion of $f(a)$ around point $x$ gives
\begin{subequations}\label{eq:taylor_f}
\begin{align}
    f(a) 
    & \approx f(x) + (a-x)f'(x) \\
    f(x-\kappa) &\approx f(x) + (x-\kappa-x) f'(x) \\
    &= f(x) - \kappa f'(x) \\
    &= x\log x - \kappa(\log x + 1)
\end{align}
\end{subequations}
This approximation is accurate whenever $\kappa \ll x$: the neglected remainder is $\frac{1}{2}\kappa^2 f''(\xi) = O(\kappa^2/x)$, which is small relative to the retained term $\kappa(\log x + 1)$ exactly in that regime.
Now, choosing $x = \countsnew{\subword^1}{\dataset}$, we get
\begin{subequations}
\begin{align}
    &f(\countsnew{\subword^1}{\dataset}-\kappa) \nonumber \\
    &\qquad  = \bigl(
    \countsnew{\subword^1}{\dataset}
    -
    \kappa
    \bigr)
    \log
    \bigl(
    \countsnew{\subword^1}{\dataset}
    -
    \kappa
    \bigr)
    \\
    &\qquad \approx
    \countsnew{\subword^1}{\dataset}
    \log
    \countsnew{\subword^1}{\dataset}
    -\kappa
    \left(
    \log \countsnew{\subword^1}{\dataset}
    + 1
    \right).
\end{align}
\end{subequations}
which implies 
\begin{align}
    &\bigl(
    \countsnew{\subword^1}{\dataset}
    -
    \kappa
    \bigr)
    \log
    \bigl(
    \countsnew{\subword^1}{\dataset}
    -
    \kappa
    \bigr) - \countsnew{\subword^1}{\dataset}
    \log
    \countsnew{\subword^1}{\dataset}
    \\
    &\qquad\qquad\qquad\qquad\qquad\quad \approx
    -\kappa
    \left(
    \log \countsnew{\subword^1}{\dataset}
    + 1
    \right). \nonumber 
\end{align}
With similar results for $x=\countsnew{\subword^2}{\dataset}$ and $x=\totcnt$.

Substituting these terms into \cref{eq:pmi_exact_start}, we get%
\begin{subequations}
\begin{align}
    \tokendelta{\subword^1,\subword^2}{\objectivefuncll(\cdot,\dataset)}
    &\approx
    -\kappa
    \left(
    \log \countsnew{\subword^1}{\dataset}
    + 1
    \right)
    -
    \kappa
    \left(
    \log \countsnew{\subword^2}{\dataset}
    + 1
    \right)
    \nonumber\\
    &\quad
    +
    \kappa\log\kappa
    +
    \kappa
    \left(
    \log \totcnt + 1
    \right)
    \\
    &=
    \kappa
    \left(
    \log
    \frac{
        \totcnt \, \kappa
    }{
        \countsnew{\subword^1}{\dataset}
        \countsnew{\subword^2}{\dataset}
    }
    - 1
    \right) \\
    &=
    \kappa
    \left(
    \operatorname{PMI}_{\mathrm{adj}}(\subword^1,\subword^2)
    - 1
    \right).
\end{align}
\end{subequations}
which completes the proof.
\end{proof}

\section{Tokeniser Training Details}
\label{app:tokeniser_details}

All four tokenisers use the same preprocessing pipeline, implemented with the HuggingFace \texttt{tokenizers} library.
Text is normalised with NFC, and pretokenised with a byte-level pretokeniser, \texttt{ByteLevel(add\_prefix\_space=False, trim\_offsets=True, use\_regex=True)}, which applies the GPT-2 pretokenisation regex.
The corresponding decoder and post-processor are:
\begin{align}
    &\texttt{ByteLevel(add\_prefix\_space=True,} \nonumber \\
    &\qquad  \texttt{trim\_offsets=True, use\_regex=True)} \nonumber \\
    &\texttt{ByteLevel(add\_prefix\_space=True,} \nonumber \\
    &\qquad  \texttt{trim\_offsets=False, use\_regex=True)}  \nonumber 
\end{align}
respectively.
This choice is deliberately standard.
Because pretokenisation strongly constrains which token boundaries a tokeniser can learn, holding it fixed is necessary for the objective and the search procedure to be the only varying factors in our comparison.
We use byte-level fallback throughout, so every string remains encodable and no \texttt{<unk>} tokens are required.

\section{Hyperparameters and Architecture}
\label{app:hyperparams}

All language models follow a Llama-style decoder-only architecture. \Cref{tab:arch_hyperparams} summarises the per-size architectural settings. All models use SiLU activations, grouped-query attention with separate $K/V$ head counts, RMSNorm, and rotary positional embeddings (RoPE, $\theta=10{,}000$).

\begin{table*}[t]
\centering
\adjustbox{max width=\textwidth}{%
\begin{tabular}{lcccccc}
\toprule
Size & Hidden & Intermediate & Heads & KV heads & Layers & Tied \\
\midrule
100M & 576  & 1{,}536 & 9  & 3 & 30 & \checkmark \\
300M & 960  & 2{,}560 & 15 & 5 & 32 & \checkmark \\
500M & 1280 & 3{,}456 & 16 & 4 & 26 & \checkmark \\
1B   & 2048 & 5{,}632 & 32 & 4 & 22 & \checkmark \\
\bottomrule
\end{tabular}}
\caption{Per-size architecture for the Llama-style language models. ``Tied'' denotes weight tying between input and output embeddings.}
\label{tab:arch_hyperparams}
\end{table*}

\paragraph{Optimisation.}
We train models with AdamW using learning rate $3\!\times\!10^{-4}$, $\beta_1\!=\!0.9$, $\beta_2\!=\!0.95$, weight decay $0.1$, and a maximum gradient norm of $1.0$. The learning-rate schedule is warmup--stable--decay with $2{,}000$ warmup steps, a stable phase, and $10{,}000$ linear-decay steps, ending with a final learning rate of $10\%$ of its peak value.

\paragraph{Training setup.}
All models are trained at sequence length $2{,}048$ in bf16-mixed precision with FlashAttention~2. The per-device batch size is $32$, and gradient accumulation is adjusted per model size to match a fixed effective batch size. Following Chinchilla-optimal scaling, each model is trained on roughly $20\times$ its parameter count in tokens. Training uses a single fixed random seed (42), except for the 100M setting and the English 1B models at 128k vocabulary, where we train three seeds (42, 43, 44) per configuration. 
All experiments are implemented in PyTorch with our open-source training stack.

\section{Detailed Metric Definitions}
\label{app:metric_defs}

This section gives precise definitions for every metric used in the paper. 
Let $\dataset$ denote a held-out corpus, $N_{\text{bytes}}$ its size in UTF-8 bytes, and $\totcnt$ the number of tokens produced by a given tokeniser on $\dataset$.

\paragraph{Extrinsic.} We evaluate models based on two extrinsic metrics.
    \defn{Bits per byte (\bpb)}, defined as $-\log_2 P_\theta(\dataset) / N_{\text{bytes}}$: the trained model's negative log-likelihood on $\dataset$, normalised by the byte count. Lower is better. Because the denominator is a property of the raw text (and hence identical across tokenisers), \bpb is directly comparable across tokenisers; token-level perplexity, by contrast, is biased against tokenisers that produce shorter sequences.
    \defn{Minimal-pair grammatical accuracy,} which, for an acceptable--unacceptable sentence pair, evaluates the model as correct if it assigns higher per-token log-probability to the acceptable sentence. It than averages this values across a corpus of such pairs. We use BLiMP \citep{warstadt2020blimp} for English, MultiBLiMP \citep{jumelet-etal-2026-multiblimp} for English, German, Spanish, and Turkish, and ZhoBLiMP \citep{liu2025systematicassessmentlanguagemodels} for Chinese. 
    For this metric, higher is better.

\paragraph{Intrinsic.} We evaluate models on 5 types of intrinsic metrics.
    \defn{Compression} is $\objectivefunccomp = \totcnt$, i.e., the corpus token count, where lower means more compact. 
    \defn{Bytes per token} (BPT)  is defined as $N_{\text{bytes}} / \totcnt$, the average number of source bytes captured by one token, where higher means more compact. 
    \defn{Vocab Util} is the fraction of vocabulary IDs that appear at least once in $\dataset$.
    \defn{Token length} (Tok Len) is the average character length of the entries in the learned vocabulary, with each entry counted once, unweighted by corpus frequency.
    (Unigram) \defn{Entropy} is defined here as $H = -\sum_{\subword} \ptheta(\subword) \log_2 \ptheta(\subword)$, where $\ptheta(\subword)$ is a unigram language model trained on token counts. 
    \defn{Unigram log-likelihood objective} is $\objectivefuncll = H \cdot \totcnt$, where lower suggests a better fit under the empirical unigram model. 
    \defn{Zipf $\alpha$} is the negated slope of a log-log OLS fit of frequency vs. rank, with tokens sorted in descending frequency; higher $\alpha$ means a steeper distribution. \defn{Coverage 50\%} is the smallest $\ell$ such that the $\ell$ most frequent tokens together account for at least half of all token occurrences in $\dataset$; smaller $\ell$ means usage is concentrated on a small core.\looseness=-1

\paragraph{Vocabulary overlap.} 
Finally, we also compute vocabulary overlap between tokenisers. 
For two tokenisers $A$ and $B$, their overlap is measured as $\nicefrac{|\vocab_A \cap \vocab_B|}{|\vocab_A|}$, where $\vocab_A$ and $\vocab_B$ are their learned vocabularies. 
This measure is symmetric for tokenisers with equal vocabulary size.

\section{English vs.\ Multilingual Intrinsic Comparison}
\label{app:english_vs_multi}

This appendix reports the multilingual counterpart of \Cref{tab:tokenizer_stats_grouped,tab:vocab_overlap}, using the multilingual tokenisers and the concatenation of the five per-language test splits (220{,}423 documents: 47{,}384 English, 54{,}969 German, 54{,}489 Spanish, 41{,}986 Turkish, 21{,}595 Chinese).\looseness=-1

\paragraph{Compression gaps widen in the multilingual setting.}
\Cref{tab:app_distribution_eng_vs_multi} shows each tokenisers compression in both English-only and multilingual corpora.
In English, \UnigramLM{} produces 18.5\% more tokens than \BPE{} and \CompMax{} 7.4\% more.
In multilingual, the top-down gap roughly doubles: \UnigramLM{} +39.6\%, \CompMax{} +26.5\%.
\GreedyLL{} remains close to \BPE{} in both settings (+0.5\% in English, +0.9\% in multilingual).
The bottom-up vs.\ top-down separation observed in the main text therefore grows under a shared multilingual vocabulary budget.

\begin{table*}[t]
\centering
\adjustbox{max width=\textwidth}{%
\begin{tabular}{lllccccccccc}
\toprule
Setting & Tokeniser & Objective & Search & $\objectivefunccomp$ $\downarrow$ & $\objectivefuncll$ $\downarrow$ & Entropy & Zipf $\alpha$ & BPT $\uparrow$ & Tok Len & Vocab Util $\uparrow$ & Coverage 50\% \\
\midrule
\multirow{4}{*}{English}
& \BPE{}        & $\objectivefunccomp$ & bottom-up & \textbf{46{,}859{,}565} & $5.216 \times 10^{8}$         & 11.132 & 1.484 & \textbf{4.886} & 7.06 & 98.9\% & 204 \\
& \CompMax{}    & $\objectivefunccomp$ & top-down  & 50{,}312{,}483          & $5.484 \times 10^{8}$         & 10.900 & 1.443 & 4.551 & 6.01 & \textbf{99.8\%} & 149 \\
& \GreedyLL{}   & $\objectivefuncll$   & bottom-up & 47{,}111{,}733          & $\bm{5.203 \times 10^{8}}$    & 11.044 & 1.574 & 4.860 & 7.48 & 96.8\% & 191 \\
& \UnigramLM{}  & $\objectivefuncll$   & top-down  & 55{,}546{,}423          & $5.552 \times 10^{8}$         &  9.996 & 1.717 & 4.122 & 6.46 & 97.1\% & 53 \\
\midrule
\multirow{4}{*}{Multilingual}
& \BPE{}        & $\objectivefunccomp$ & bottom-up & \textbf{172{,}372{,}100} & $2.262 \times 10^{9}$        & 13.122 & 1.247 & \textbf{4.675} & 6.96 & 99.92\% & 1{,}483 \\
& \CompMax{}    & $\objectivefunccomp$ & top-down  & 217{,}983{,}502          & $2.678 \times 10^{9}$        & 12.286 & 1.266 & 3.696 & 5.52 & \textbf{99.94\%} & 618 \\
& \GreedyLL{}   & $\objectivefuncll$   & bottom-up & 173{,}917{,}215          & $\bm{2.245 \times 10^{9}}$   & 12.907 & 1.347 & 4.633 & 7.40 & 99.36\% & 1{,}114 \\
& \UnigramLM{}  & $\objectivefuncll$   & top-down  & 240{,}669{,}745          & $2.597 \times 10^{9}$        & 10.791 & 1.598 & 3.348 & 6.79 & 99.74\% & 139 \\
\bottomrule
\end{tabular}}
\caption{Intrinsic statistics in English-only and multilingual settings.
All values are computed at 128k vocabulary size. The multilingual evaluation corpus is the union of the five per-language test splits.}
\label{tab:app_distribution_eng_vs_multi}
\end{table*}

\paragraph{Distributional shape changes with the evaluation corpus.}
\Cref{tab:app_distribution_eng_vs_multi} reports the full intrinsic statistics in both settings.
All four token distributions flatten in the multilingual setting: Zipf $\alpha$ values decrease, entropies increase, and Coverage 50\% grows by an order of magnitude (e.g.,\ \BPE{}: 204 $\to$ 1{,}483), reflecting the broader set of frequent token types in a five-language corpus.
The relative pattern between objectives is preserved, however: \UnigramLM{} retains the steepest Zipf $\alpha$ and lowest entropy in both settings, and likelihood-based methods stay more concentrated than their compression-based counterparts within each search family.

\paragraph{Vocabulary utilisation rises with corpus breadth.}
The unused-vocabulary tail visible in English (Vocab Util 96.8--99.8\%) largely disappears in the multilingual evaluation (99.36--99.94\%), because the broader, more diverse corpus exercises a larger portion of the 128k vocabulary.

\paragraph{Vocabulary overlap follows the same cluster structure.}
\Cref{tab:app_vocab_overlap_multi} is the multilingual counterpart of \Cref{tab:vocab_overlap}.
\BPE{} and \GreedyLL{} remain the most similar pair (80.5\% overlap); \CompMax{} and \UnigramLM{} form a second cluster (63.6\%); cross-family overlap stays in the 30--37\% range.\looseness=-1

\begin{table}[t]
\centering
\adjustbox{max width=\linewidth}{%
\begin{tabular}{lcccc}
\toprule
 & \BPE{} & \CompMax{} & \GreedyLL{} & \UnigramLM{} \\
\midrule
\BPE{}        & ---    & 37.0\% & 80.5\% & 32.3\% \\
\CompMax{}    & 37.0\% & ---    & 33.6\% & 63.6\% \\
\GreedyLL{}   & 80.5\% & 33.6\% & ---    & 31.2\% \\
\UnigramLM{}  & 32.3\% & 63.6\% & 31.2\% & ---    \\
\bottomrule
\end{tabular}}
\caption{Multilingual vocabulary overlap at 128k vocabulary size; the multilingual counterpart of \Cref{tab:vocab_overlap}.}
\label{tab:app_vocab_overlap_multi}
\end{table}

\section{Analysis of Merge Dynamics}
\label{app:merge_dynamics}

Merge trajectories are only defined for the bottom-up methods, so this analysis covers \BPE{} and \GreedyLL{} alone.
A token is \defn{absorbed} if it is created by one merge and later appears inside another merge: if a tokeniser first builds \texttt{physi} and then merges it into \texttt{physic}, \texttt{physi} has been absorbed.
Absorption is common and not by itself a sign of waste, since an absorbed token usually remains in the vocabulary and can still used on its own.
It only becomes wasteful when the absorbed token reduces in frequency enough to be considered not-useful any more.\looseness=-1

\begin{table}[t]
\centering
\adjustbox{max width=\columnwidth}{%
\begin{tabular}{lccc}
\toprule
Metric & \BPE{} & \GreedyLL{}  & \GreedyLL{}  \\
&& (Exact) & (Approx) \\
\midrule
Total absorbed & 34,224 & 41,599 & 41,751 \\
Absorption rate (all) & 26.8\% & 32.6\% & 32.7\% \\
Absorption rate ($|\subword|\geq 5$) & 21.7\% & 27.8\% & 27.9\% \\
Absorption rate ($|\subword|\geq 10$) & 4.5\% & 8.5\% & 8.8\% \\
\midrule
Absorbed with zero usage & 3.5\% & 9.2\% & 8.7\% \\
Absorbed with low usage ($\leq 10$) & 26.2\% & 36.9\% & 36.0\% \\
Not absorbed, zero usage & .2\% & .4\% & .4\% \\
\bottomrule
\end{tabular}}
\caption{Merge-dynamics for bottom-up methods. 
Low usage means at most 10 occurrences on the test set. 
Absorbed with zero or low usage are given as a percentage of total absorptions.}
\label{tab:kill_dynamics_summary}
\end{table}

\Cref{tab:kill_dynamics_summary} present absorption rates for both \BPE and \GreedyLL.
Absorption is common for both methods, and by itself says little. 
The methods separate, however, on the tokens that are absorbed and then never used at all, which \GreedyLL{} produces nearly three times as often as \BPE{} (9.2\% vs.\ 3.5\%). 
This is where its lower vocabulary utilisation at 128k comes from.
Also interestingly, the absorption gap widens for long tokens (8.5\% vs.\ 4.5\% at ten characters or more). 
This suggests \GreedyLL reaches its final vocabulary through more intermediate steps than \BPE, which is how it ends up with longer tokens.

\begin{figure*}[t]
\centering
\includegraphics[width=0.82\textwidth]{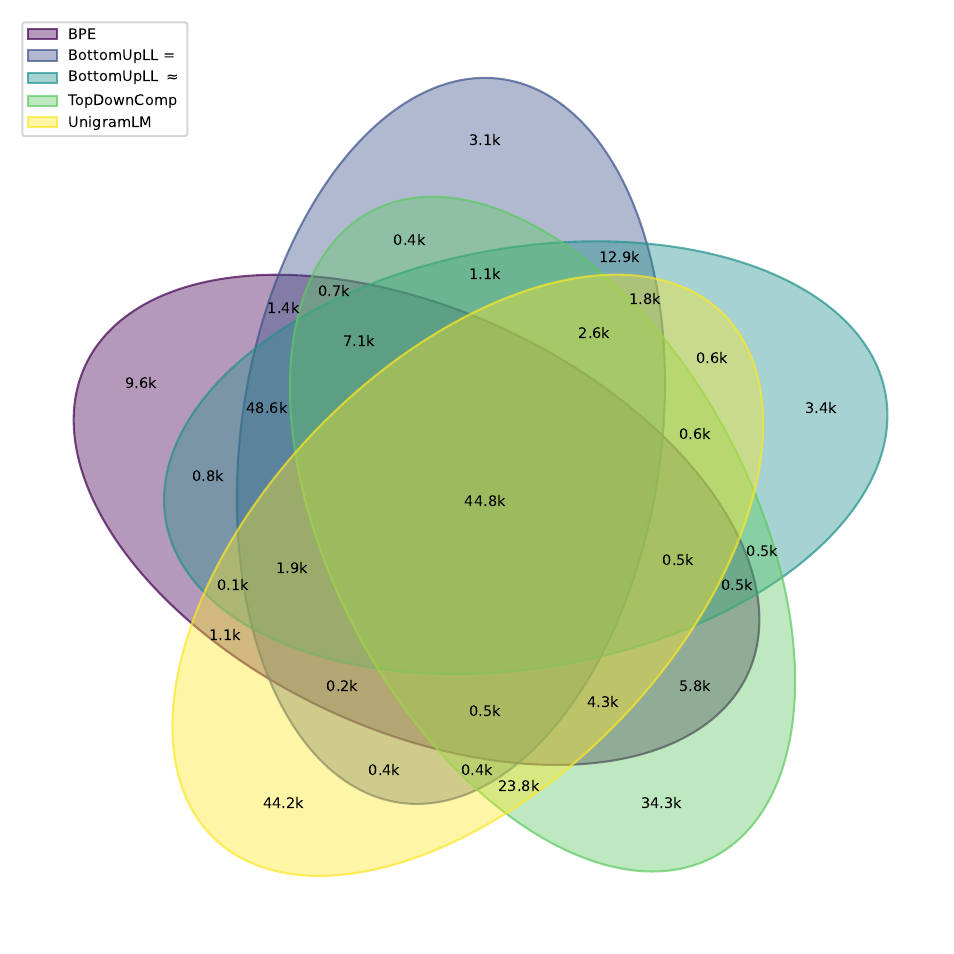}
\vspace{-20pt}
\caption{Five-set Venn diagram of the 128k vocabularies. Each region gives the number of tokens (in thousands) belonging to exactly that combination of tokenisers. The largest two-way region is \BPE{} $\cap$ \GreedyLL{}, and the \GreedyLL{} exact and approximate variants are nearly coextensive.}
\label{fig:venn5}
\end{figure*}

\section{Vocabulary Overlap Structure}
\label{app:vocab_venn}

\Cref{tab:vocab_overlap} reports pairwise vocabulary overlap in English, but pairwise numbers cannot show how the five vocabularies intersect jointly.
\Cref{fig:venn5} gives the full picture as a five-set Venn diagram over the 128k English-only vocabularies.
The structure mirrors the pairwise results: the largest regions are those shared by the two \greedyalg{} methods, the two \pruningalg{} methods retain the largest exclusive regions, and cross-family regions are comparatively small.

\section{Qualitative Examples}
\label{app:qual_ex}

Aggregate metrics show the dominant trends, but they hide the local segmentation decisions that produce those trends. 
Here, we inspect a small set of representative examples covering plain English, scientific terminology, code-like text, URLs, Chinese, and multilingual accented text. 
These examples are illustrative rather than statistically decisive, but they make the inductive biases of the methods easier to interpret.
\Cref{tab:example_token_counts} reports the total number of tokens produced across the 10 illustrative examples. 
The ordering broadly matches the full compression benchmark.

\begin{table}[t]
\centering
\adjustbox{max width=\columnwidth}{%
\begin{tabular}{lc}
\toprule
Method & \# of Tokens $\downarrow$ \\
\midrule
\BPE{}               & 156 \\
\CompMax{}           & 160 \\
\GreedyLL{} (Exact)  & 154 \\
\GreedyLL{} (Approx) & \textbf{153} \\
\UnigramLM{}         & 173 \\
\bottomrule
\end{tabular}}
\caption{Total number of tokens produced across 10 illustrative test examples. Lower is better.}
\label{tab:example_token_counts}
\end{table}

\paragraph{Scientific and medical terminology.}
On examples such as: \\
(i) \texttt{Photosynthesis converts carbon dioxide and water into glucose and oxygen.} \\
or \\
(ii) \texttt{Electroencephalography is used to diagnose neurological disorders.} \\
\GreedyLL{} often produces the shortest segmentation. It tends to preserve long technical units such as \texttt{Photosynthesis} or fragments of \texttt{encephalography} as large tokens. \BPE{} is also compact, but more often splits these words into high-frequency pieces. 
This is consistent with the merge-dynamics analysis: \GreedyLL{} is more willing to build long lexical chains through intermediate tokens.

\paragraph{Code and punctuation-heavy text.}
On code-like or SQL-like strings, \BPE{} is usually strongest. For example, on: \\
(iii) \texttt{SELECT * FROM users WHERE id = 1; DROP TABLE users;--} \\
\BPE{} produces the fewest tokens and preserves punctuation patterns such as \texttt{;--}. \UnigramLM{} is much more fragmented, often splitting short code tokens such as \texttt{id}, \texttt{if}, or \texttt{def} into character-level tokens. This reflects the concentration of its token distribution: unless short code tokens are strongly supported by the learned vocabulary, the model falls back to smaller tokens.

\paragraph{URL-like text.}
For URL strings such as: \\
(iv) \url{https://www.example.com/path/to/resource?query=value&page=1} \\
\BPE{} and \CompMax{} are especially compact. \BPE{} often preserves web-specific substrings such as \texttt{https}, while \GreedyLL{} may split them into statistically meaningful but less URL-specific pieces such as \texttt{http}+\texttt{s}. This is a case where frequency-driven compression appears better matched to the surface regularities of web text than likelihood-based lexical cohesion.

\paragraph{Chinese text.}
On the short Chinese example: \\
(v)  \begin{CJK*}{UTF8}{gbsn}\texttt{人工智能正在改变世界}\end{CJK*} \\
\UnigramLM{} produces the fewest tokens among the inspected methods. 
This does not contradict the multilingual BPB results in \Cref{tab:main_results_multi}, where bottom-up methods are much better on Chinese overall.
Rather, it shows that example-level token count and corpus-level language-model BPB can diverge: a tokeniser may preserve some non-Latin chunks well while still producing a worse distribution of tokens for language modelling across the full evaluation set.

\paragraph{Multilingual accented text.}
Finally, for: \\
(vi) \texttt{¡Hola! ¿Cómo estás? Très bien, merci. Danke schön!} \\
all methods produce similar total counts, but the internal boundaries differ. \UnigramLM{} tends to preserve accented fragments such as \texttt{Cómo} and \texttt{estás} more coherently, while \BPE{} more often splits them into smaller pieces. Thus, even when example-level token counts are similar, tokenisers can impose different boundaries that may matter for downstream behaviour.

\paragraph{Take away.}
Together, these examples illustrate the same qualitative biases seen in the aggregate analyses: \BPE{} is robust on punctuation-heavy and web-like text, \GreedyLL{} tends to form longer lexical units, \UnigramLM{} can preserve some accented or non-Latin fragments while fragmenting other domains, and \CompMax{} behaves as a compression-oriented top-down method.

\section{Significance of \bpb{} Differences}
\label{app:significance}

Some \bpb{} differences between tokenisers are small, so we test whether they are robust to the choice of evaluation documents.
For each objective-matched pair of tokenisers, we run a paired document-level bootstrap over the held-out English test set ($47{,}384$ documents).
In each of $10{,}000$ iterations we resample documents with replacement and use the \emph{same} resampled indices for both models, so that the shared per-document difficulty cancels.
We then recompute each model's corpus-level \bpb{} on the resampled set and record the difference between the \greedyalg{} and \pruningalg{} model.
\Cref{tab:bootstrap} reports point estimates and $95\%$ confidence intervals for the 128k-vocabulary models, using seed 42 for each configuration.

\begin{table}[h]
\centering
\adjustbox{max width=\columnwidth}{%
\begin{tabular}{llrc}
\toprule
Size & Comparison & $\Delta$ \bpb{} & 95\% CI \\
\midrule
\multirow{2}{*}{300M}
 & \BPE{} vs.\ \CompMax{}        & $-0.0031$ & $[-0.0032, -0.0029]$ \\
 & \GreedyLL{} vs.\ \UnigramLM{} & $-0.0080$ & $[-0.0082, -0.0078]$ \\
\midrule
\multirow{2}{*}{500M}
 & \BPE{} vs.\ \CompMax{}        & $-0.0016$ & $[-0.0017, -0.0014]$ \\
 & \GreedyLL{} vs.\ \UnigramLM{} & $-0.0075$ & $[-0.0077, -0.0074]$ \\
\midrule
\multirow{2}{*}{1B}
 & \BPE{} vs.\ \CompMax{}        & $-0.0025$ & $[-0.0027, -0.0023]$ \\
 & \GreedyLL{} vs.\ \UnigramLM{} & $-0.0039$ & $[-0.0041, -0.0037]$ \\
\bottomrule
\end{tabular}}
\caption{Paired document-level bootstrap ($10{,}000$ resamples) of \bpb{} differences between objective-matched tokenisers at 128k vocabulary. Negative values indicate the \greedyalg{} method achieves lower \bpb{}. All confidence intervals exclude zero.}
\label{tab:bootstrap}
\end{table}

All six confidence intervals lie entirely below zero, indicating that, for each objective, the \greedyalg{} tokeniser achieves significantly lower \bpb{} than its \pruningalg{} counterpart, and that this ordering is not an artifact of the particular held-out documents used for evaluation.

\paragraph{Training-seed variance.}
The bootstrap above controls for the choice of evaluation documents, but not for randomness in model training.
To assess the latter, we retrain the English 1B models at 128k vocabulary with two additional seeds (43 and 44), giving three seeds per tokeniser.
\Cref{tab:seeds} reports the per-seed \bpb{} values.
Both objective-matched orderings are preserved in the mean: \BPE{} outperforms \CompMax{}, and \GreedyLL{} outperforms \UnigramLM{}.
We note that \GreedyLL{} shows a noticeably larger spread than the other tokenisers, driven by a single seed; averaged over three seeds it attains a lower \bpb{} than \CompMax{}, whereas on seed 42 alone the ordering is reversed.

\begin{table}[h]
\centering
\adjustbox{max width=\columnwidth}{%
\begin{tabular}{lcccl}
\toprule
& \multicolumn{3}{c}{Seed} & \\\cmidrule(lr){2-4}
Tokeniser & 42 & 43 & 44 & Mean \\
\midrule
\BPE{}       & 0.7795 & 0.7787 & 0.7787 & 0.7790 \std{4e-4} \\
\CompMax{}   & 0.7819 & 0.7817 & 0.7813 & 0.7816 \std{3e-4} \\
\GreedyLL{}  & 0.7833 & 0.7789 & 0.7787 & 0.7803 \std{26e-4} \\
\UnigramLM{} & 0.7872 & 0.7871 & 0.7873 & 0.7872 \std{1e-4} \\
\bottomrule
\end{tabular}}
\caption{Per-seed \bpb{} for the English 1B models at 128k vocabulary. Mean is taken over the three seeds, with standard deviation.}
\label{tab:seeds}
\end{table}

\section{Effect of the Top-Down Approximations}
\label{app:prune_rate}

In this section, we investigate the impact of our approximations in \pruningalg methods.

\paragraph{Pruning rate.}
For efficiency and as is standard practice, our \pruningalg{} methods remove batches of tokens at each pruning round, rather than a single token.
Specifically, our main experiments prune $10\%$ of the current vocabulary per round.
To assess how much this batching affects the learned vocabulary, we re-train \UnigramLM{} and \CompMax{} at 128k vocabulary using pruning rates of $1\%$ and $0.1\%$ per round, and measure the vocabulary overlap with the corresponding $10\%$ tokeniser.
For \UnigramLM{}, the overlap is $97.3\%$ at a $1\%$ pruning rate and $97.1\%$ at $0.1\%$; for \CompMax{} it remains above $99\%$ in both cases.
Inspecting the non-overlapping tokens, we find that they occur very rarely in the tokenised corpus.
Pruning in batches therefore appears to have a limited effect on the resulting vocabulary.

\paragraph{Local replacement.}
Our second approximation is the local replacement approximation of \cref{sec:method_topdown}, which scores a deletion by re-segmenting only the deleted token, rather than recomputing the objective over the whole corpus.
Computing exact deletion costs requires re-segmenting the corpus once per candidate token at every pruning step, which is intractable at the vocabulary sizes used in our main experiments.
We therefore compare the two at a small scale: on $5{,}000$ lines of English FineWeb, we prune from a seed vocabulary of $3{,}000$ tokens down to $800$ with a pruning rate of $10\%$, once using exact deletion costs and once using the local replacement approximation.
The two procedures agree on $652$ of the $800$ resulting tokens, an overlap of $81.5\%$.
The approximation therefore recovers most, but not all, of the vocabulary that exact scoring would select.
We note that this is a small-vocabulary and small-corpus setting, though, and that the effect of the approximation at the scales used in our main experiments remains open.

\end{document}